\documentclass[twoside,11pt]{article}

\pdfoutput=1

\usepackage{lastpage}
\usepackage{jmlr2e}

\jmlrheading{20}{2019}{1--\pageref{LastPage}}{6/18; Revised 4/19}{10/19}{18-399}{Niklas Pfister$^*$, Sebastian Weichwald$^*$, Peter B{\"u}hlmann, and Bernhard Sch{\"o}lkopf}

\ShortHeadings{coroICA}{Pfister$^*$, Weichwald$^*$, B{\"u}hlmann, and Sch{\"o}lkopf}
\firstpageno{1}

\usepackage[utf8]{inputenc}
\usepackage[english]{babel}
\usepackage{amsmath,amssymb,amsfonts,mathrsfs}
\usepackage{mathtools}
\usepackage{dsfont}
\usepackage{hyperref}
\usepackage{enumerate}
\usepackage{paralist}
\usepackage{paralist}
\usepackage{graphicx}
\usepackage{subcaption}
\usepackage{gensymb}
\usepackage[dvipsnames]{xcolor}
\usepackage[framemethod=TikZ]{mdframed}
\usepackage[ruled]{algorithm2e}
\usepackage{tikz}
\usepackage[toc,page,header]{appendix}
\usepackage{tkz-graph}
\tikzset{node/.style={circle, draw, minimum size=2.5em, text centered, decorate, thick, scale=.85}}
\tikzset{nodec/.style={circle, draw, minimum size=2.5em, text centered, decorate, thick, fill=lightgray, scale=.85}}
\tikzset{nodeh/.style={circle, dashed, draw, minimum size=2.5em, text centered, decorate, thick, scale=.85}}
\usetikzlibrary{snakes}
\tikzset{waves/.style={snake=expanding waves,segment length=2pt,lightgray}}
\usetikzlibrary{shapes.geometric}
\usetikzlibrary{backgrounds}

\graphicspath{{./plots/}}

\newcommand{\R}{\mathbb{R}}
\newcommand{\N}{\mathbb{N}}

\DeclarePairedDelimiterX{\norm}[1]{\lVert}{\rVert}{#1}
\DeclarePairedDelimiterX{\abs}[1]{\lvert}{\rvert}{#1}
\renewcommand{\epsilon}{\varepsilon}

\newcommand{\E}{\mathbb{E}}

\newcommand{\var}{\operatorname{Var}}
\newcommand{\cov}{\operatorname{Cov}}

\newcommand{\empcov}{\widehat{\operatorname{Cov}}}
\newcommand{\empcorr}{\widehat{\operatorname{Corr}}}

\newcommand{\identity}{\operatorname{Id}}

\newcommand{\vH}{\mathbf{H}}

\newcommand{\vS}{\mathbf{S}}

\newcommand{\vX}{\mathbf{X}}

\newcommand{\groups}{\mathcal{G}}

\newcommand{\COtwo}{\log(\operatorname{CO}_2)}
\newcommand{\Temp}{\operatorname{T}}

\newtheorem{assumption}{Assumption}
\newcommand\Tau{\mathrm{T}}

\usepackage[labelsep=period]{caption}
\usepackage{float}

\usepackage{xspace}

\newcommand\coroICA{\textsf{coroICA}\xspace}
\newcommand\coroICATD{\textsf{coroICA\,(TD)}\xspace}
\newcommand\coroICAvar{\textsf{coroICA\,(var)}\xspace}
\newcommand\coroICAvarTD{\textsf{coroICA\,(var\,\&\,TD)}\xspace}

\newcommand\oururl{\url{https://sweichwald.de/coroICA/}\xspace}

\usepackage{placeins}

\definecolor{colA}{HTML}{008080}
\definecolor{colB}{HTML}{67001F}
\definecolor{colC}{HTML}{449B52}
\definecolor{colD}{HTML}{979797}

\hypersetup{
    colorlinks,
    linkcolor={colC},
    citecolor={colB},
    urlcolor={colA}
}

\usepackage{xpatch}
\makeatletter
\xpatchcmd{\endmdframed}
  {\aftergroup\endmdf@trivlist\color@endgroup}
  {\endmdf@trivlist\color@endgroup\@doendpe}
  {}{}
\makeatother

\usepackage{scrextend}
\makeatletter
\def\maketitle{\par
 \begingroup
   \def\thefootnote{\fnsymbol{footnote}}
   \def\@makefnmark{\hbox to 0pt{$^{\@thefnmark}$\hss}}
   \deffootnote[1.7em]{1.6em}{2em}{$^\thefootnotemark$}
   \@maketitle \@thanks
 \endgroup
\setcounter{footnote}{0}
 \let\maketitle\relax \let\@maketitle\relax
 \gdef\@thanks{}\gdef\@author{}\gdef\@title{}\let\thanks\relax}
\makeatother

\begin{document}

\title{Robustifying Independent Component Analysis\\by Adjusting for Group-Wise Stationary Noise}

\author{\name Niklas Pfister\footnotemark[1]\phantom{\thanks{Authors contributed equally. Most of this work was done while SW was at the Max Planck Institute for Intelligent Systems, Tübingen, Germany.}} \email pfister@stat.math.ethz.ch \\
       \addr Seminar for Statistics, ETH Zürich\\
       Rämisstrasse 101, 8092 Zürich, Switzerland
       \AND
       \name Sebastian Weichwald\footnotemark[1] \email sweichwald@math.ku.dk \\
       \addr Department of Mathematical Sciences, University of Copenhagen\\
       Universitetsparken 5, 2100 Copenhagen, Denmark
       \AND
       \name Peter Bühlmann \email buhlmann@stat.math.ethz.ch \\
       \addr Seminar for Statistics, ETH Zürich\\
       Rämisstrasse 101, 8092 Zürich, Switzerland
       \AND
       \name Bernhard Schölkopf \email bs@tue.mpg.de \\
       \addr Empirical Inference Department, Max Planck Institute for Intelligent Systems\\
       Max-Planck-Ring 4, 72076 Tübingen, Germany}

\editor{Kenji Fukumizu}
\maketitle

\begin{abstract}%
  We introduce \coroICA, confounding-robust independent component
  analysis, a novel ICA algorithm which decomposes linearly mixed
  multivariate observations into independent components that are
  corrupted (and rendered dependent) by hidden group-wise stationary
  confounding. It extends the ordinary ICA model in a theoretically
  sound and explicit way to incorporate group-wise (or
  environment-wise) confounding. We show that our proposed general
  noise model allows to perform ICA in settings where other noisy ICA
  procedures fail. Additionally, it can be used for applications with
  grouped data by adjusting for different stationary noise within each
  group. Our proposed noise model has a natural relation to causality
  and we explain how it can be applied in the context of causal
  inference. In addition to our theoretical framework, we provide an
  efficient estimation procedure and prove identifiability of the
  unmixing matrix under mild assumptions. Finally, we illustrate the
  performance and robustness of our method on simulated data, provide
  audible and visual examples, and demonstrate the applicability to
  real-world scenarios by experiments on publicly available Antarctic
  ice core data as well as two EEG data sets.  We provide a
  scikit-learn compatible pip-installable Python package \coroICA as
  well as R and Matlab implementations accompanied by a documentation
  at \oururl.
\end{abstract}


\section{Introduction}

The analysis of multivariate data is often complicated by high
dimensionality and complex inter-dependences between the observed
variables. In order to identify patterns in such data it is therefore
desirable and often necessary to
separate different aspects of the data.  In multivariate statistics,
for example, principal component analysis~(PCA) is a common
preprocessing step that decomposes the data into orthogonal principle
components which are sorted according to how much variance of
the original data each component explains. There are two important
applications of this. Firstly, one can reduce the
dimensionality of the data by projecting it onto the lower
dimensional space spanned by the leading principal components which
maximize the explained variance. Secondly, since the principle
components are orthogonal, they separate in some sense different
(uncorrelated) aspects of the data. In many situations this enables a
better interpretation and representation.

Often, however, PCA may not be sufficient to separate the data in a
desirable way due to more complex inter-dependences in the
multivariate data (see e.g., Section~1.3.3 in \citet{hyvarinen2001a}
for an instructive example). This observation motivates the
development of independent component analysis (ICA), formally
introduced in its current form by \citet{cardoso89} and
\citet{comon1994}. ICA is a widely used unsupervised blind source
separation technique that aims at decomposing an observed mixture of
independent source signals. More precisely, assuming that the observed
data is a linear mixture of underlying independent variables, one
seeks the unmixing matrix that maximizes the independence between the
signals it extracts. There has been a large amount of research on
different types of ICA procedures and their interpretations, e.g.,\
\citet[Infomax]{bell1995} who maximize the entropy,
\citet[fastICA]{hyvarinen1999} maximizing the kurtosis or
\citet[SOBI]{belouchrani1997} who propose to minimize time-lagged
dependences, to name only some of the widespread examples.

ICA has applications in many fields, for example in finance
\citep[e.g.,][]{back1997}, the study of functional magnetic resonance
imaging (fMRI) data
\citep[e.g.,][]{mckeown1998spatially,mckeown1998analysis,calhoun2003ica},
and notably in the analysis of electroencephalography (EEG) data
\citep[e.g.,][]{makeig1995,makeig1997blind,delorme2004eeglab}.  The
latter is motivated by the common assumption that the signals recorded
at EEG electrodes are a (linear) superposition of cortical dipole
signals \citep{Nunez2006}. Indeed, ICA-based preprocessing has become
the de facto standard for the analysis of EEG data. The extracted
components are interpreted as corresponding to cortical sources
\citep[e.g.,][]{ghahremani1996independent,zhukov2000independent,makeig2002dynamic}
or used for artifact removal by dropping components that are dominated
by ocular or muscular activity
\citep[e.g.,][]{jung2000removing,delorme2007enhanced}.

In many applications, the data at hand is heterogeneous and parts of
the samples can be grouped by the different settings (or environments)
under which the observations were taken. For example, we can group those
samples of a multi-subject EEG recording that belong to the same
subject. For the analysis and interpretation of such data across
different groups, it is desirable to extract one set of common
features or signals instead of obtaining individual ICA decompositions
for each group of samples separately. Here, we present a novel,
methodologically sound framework that extends the ordinary ICA model,
respects the group structure and is robust by explicitly accounting
for group-wise stationary confounding. More precisely, we consider a
model of the form
\begin{equation}
  \label{eq:basic_model}
  X_i = A\cdot S_i+H_i,
\end{equation}
where $i$ denotes the sample index, $A$ remains fixed across different
groups, $S_i$ is a vector of independent source signals and $H_i$ is a
vector of stationary confounding noise variables with fixed covariance
within each group (an intuitive
example where such a scenario may be encountered in practice is
illustrated in Figure~\ref{fig:eegscenario}).  Based on this extension
to ordinary ICA, we construct a method and an easy to implement
algorithm to extract one common set of sources that are robust against
confounding within each group and can be used for across-group
analyses. The unmixing also generalizes to previously unseen groups.

\subsection{Relation to Existing Work}

ICA is well-studied with a tremendous amount of research related to
various types of extensions and relaxations of the ordinary ICA
model. In light of this, it is important to understand where our
proposed procedure is positioned and why it is an interesting and
useful extension. Here, we look at ICA research from three
perspectives and illustrate how our proposed \coroICA methodology
relates to existing work. First off, in
Section~\ref{sec:classical_ica} we compare our proposed methodology
with other noisy ICA models. In Section~\ref{sec:ajd}, we review ICA
procedures based on approximate joint matrix diagonalization. Finally, in
Section~\ref{sec:grouped_data_icas} we summarize the existing
literature on ICA procedures for grouped data and highlight the
differences to \coroICA.

\subsubsection{Noisy ICA Models}\label{sec:classical_ica}

The ordinary ICA model assumes that the observed process $X$ is a
linear mixture of independent source signals $S$ \emph{without} a
confounding term $H$. Identifiability of the source signals $S$ is
guaranteed by assumptions on $S$ such as non-Gaussianity or specific
time structures. For \coroICA we require---similar to other
second-order based methods (cf.\ Section~\ref{sec:ajd})---that the
source process $S$ is non-stationary. More precisely, we require that
either the variance or the auto-covariance of $S$ changes across
time. An important extension of the ordinary ICA model is known as
noisy ICA \citep[e.g.,][]{moulines1997} in which the data generating
process is assumed to be an ordinary ICA model with additional
additive noise. In general, this leads to further identifiability
issues. These can be resolved by assuming that the additive noise is
Gaussian and the signal sources non-Gaussian
\citep[e.g.,][]{hyvarinen1999fast}, which enables correct
identification of the mixing matrix. Another possibility is to assume
that the noise is independent over time, while the source signals are
time-dependent\footnote{%
Autocorrelated signals are time-dependent, while the absence of autocorrelation does not necessarily imply time-independence of the signal. We thus use the terms time-dependence and time-independence throughout this article.}
\citep[e.g.,][]{choi2000b}. In contrast, our assumption
on the noise term $H$ is much weaker, since we only require it to be
stationary and hence in particular allow for time-dependent noise in
\coroICA. As we show in our simulations in Section~\ref{sec:GARCH}
this renders our method robust with respect to confounding noise:
\coroICA is more robust against time-dependent noise while remaining
competitive in the setting of time-independent noise. We refer to the
book by \citet{hyvarinen2001a} for a review of most of the existing ICA models
and the assumptions required for identifiability.

\subsubsection{ICA based on Approximate Joint Diagonalization}\label{sec:ajd}

As an extension of PCA, the concept of ICA is naturally connected
to the notion of joint diagonalization of covariance-type
matrices. One of the first procedures for ICA was FOBI introduced by
\citet{FOBI}, which aims to jointly diagonalize the covariance matrix
and a fourth order cumulant matrix. Extending
on this idea \citet{cardoso1993} introduced the method JADE which
improves on FOBI by diagonalizing several different fourth order
cumulant matrices. Unlike FOBI, JADE uses a general joint matrix
diagonalization algorithm which is the de facto standard for all
modern approaches. In fact, there is a still-active field that
focuses on approximate joint matrix diagonalization, commonly
restricted to positive semi-definite matrices, and often with the
purpose of improving ICA procedures \citep[e.g.,][]{cardoso1996jacobi,
  ziehe2004, tichavsky2009,ablin2018beyond}.

Both JADE and FOBI are based on the assumption that the signals are
non-Gaussian. This ensures that the sources
are identifiable given independent and identically distributed observations.
A different stream of ICA research departs
from this assumption and instead assumes that the data are a linear mixture
of independent weakly stationary time-series. This model is
often referred to as a second-order source-separation model (SOS). The
time structure in these models allows to identify the sources by
jointly diagonalizing the covariance and auto-covariance. The
first method developed for this setting is AMUSE by \citet{tong1990}
who diagonalize the covariance matrix and the auto-covariance matrix
for one fixed lag. The performance of AMUSE is, however, fragile with
respect to the exact choice of the lag, which complicates practical application \citep{miettinen2012statistical}.
Instead of only using a single lag,
\citet{belouchrani1997} proposed the method SOBI which uses all lags
up to a certain order and jointly diagonalizes all the resulting
auto-covariance matrices. SOBI is to date still one of the most
commonly employed ICA methods, in particular in EEG analysis.

The SOS model is based on the assumption of weak stationarity of the
sources which implies that the signals have fixed
variance and auto-covariance structure across time. This assumption
can be dropped and the resulting models are often termed
non-stationary source separation models (NSS). The non-stationarity
can be leveraged to boost the performance of ICA methods in various
ways \citep[see][]{matsuoka1995, hyvarinen2001cov, choi2000a,
  choi2000b, choi2001, choi2001b, pham2001}. All aforementioned
methods make use of the non-stationarity by jointly diagonalizing
different sets of covariance or auto-covariance matrices and mainly
differ by how they perform the approximate joint matrix
diagonalization. For example, the methods introduced by
\citet{choi2000a, choi2000b, choi2001} make use of non-stationarity
across sources by separating the data into blocks and jointly
diagonalizing either the covariance matrices, the auto-covariances or
both across all blocks. For our experimental comparisons, we
implemented all three of these methods with the slight modification
that we use the recent uwedge approximate joint matrix diagonalization
procedure due to \citet{tichavsky2009}. We denote the resulting three
ICA variants as
\begin{compactitem}
\item choiICA\,(var): jointly diagonalize blocks of covariances,
\item choiICA\,(TD): jointly diagonalize blocks of auto-covariances,
\item choiICA\,(var\,\&\,TD): jointly diagonalize blocks of covariances
  and auto-covariances.
\end{compactitem}
Depending on the type of matrix which is diagonalized, each procedure
detects different types of signals and behaves differently
with respect to noise. \citet{choi2001b} suggest a modification
of choiICA\,(TD) in which instead of auto-covariance matrices,
differences of auto-correlation matrices are diagonalized. The
advantage being that it captures the non-stationarity of a signal more
explicitly. Our proposed method similarly aims to use this type of
signal but instead of considering the noise-free case, we explicitly
formalize a model class that generalizes to noisy settings.
Furthermore, we provide an identifiability theorem allowing for group-wise stationary confounding.
Such a result has not been proven for the aforementioned method in the noise-free case.
For a detailed description of both
SOS- and NSS-based methods we refer the reader to the review by
\citet{Nordhausen2014} and for recent developments on leveraging
non-stationarity for identifiability in non-linear ICA see
\cite{hyvarinen2016unsupervised}.

An exhaustive comparison of all methods is infeasible on the one hand
due to the sheer amount of different models and methods and on the
other hand due to the fact that appropriately maintained and easy
adaptable code---for most methods---simply does not exist. Therefore,
we focus our comparison on the following representative, modern
methods that are most closely related to \coroICA: fastICA, SOBI, choiICA\,(TD),
choiICA\,(var), choiICA\,(TD\,\&\,var).  The methods and their respective
assumptions on the source and noise characteristics are summarized in
Table~\ref{table:ica_methods}.

\begin{table}[]
  \begin{tabular}{lll}
    \textbf{method} & \textbf{signal type} & \textbf{allowed noise} \\ \hline
    choiICA\,(TD) & varying time-dependence & time-independent \\
    choiICA\,(var) & varying variance & none \\
    choiICA\,(var\,\&\,TD) & varying variance and time-dependence & none\\
    SOBI & fixed time-dependence & time-independent\\
    fastICA\footnotemark & non-Gaussian & none \\ \hline
    \coroICA & varying time-dependence \emph{and/or} variance & group-wise stationary
  \end{tabular}
  \caption{Important ICA procedures and the signal types
    they require as well as the noise they can deal with. \coroICA is
    a confounding-robust ICA variant and is the only method
    for which an identifiability result under time-dependent noise is available.}
  \label{table:ica_methods}
\end{table}
\footnotetext{The fastICA method can be extended to include
  Gaussian noise \citep[see][]{hyvarinen1999fast}.}

\subsubsection{ICA Procedures for Grouped Data}\label{sec:grouped_data_icas}

Applications in EEG and fMRI have motivated the development of a wide
variety of blind source separation techniques which are capable of
dealing with grouped data, e.g., where groups correspond to different
subjects or recording sessions.  A short review is given in
\citet{hyvarinen2013} and a detailed exposition in the context of
fMRI data is due to \citet{calhoun2003ica}.

Consider we are given $m$ groups $\{g_1,\dots,g_m\}$
and observe a corresponding data matrix $\vX_{g_i}\in\R^{d\times n_i}$
for each group, where $d$ is the number of observed signals and $n_i$
the number of observations. Using this notation, all existing ICA procedures
for grouped data can be related to one of three underlying models
extending the classical mixing model $\vX=A\cdot \vS$. The first,
often also referred to as ``temporal concatenation'',
assumes that the mixing remains equal while the sources are allowed to
change across groups leading to data of the form
\begin{equation}
  \label{eq:time_conc}
  \left(\vX_{g_1},\dots,\vX_{g_m}\right)=A\cdot\left(\vS_{g_1},\dots,\vS_{g_m}\right).
\end{equation}
The second model, often also referred to as ``spatial concatenation'',
assumes the sources remain fixed ($n_1 = \dots = n_m$)
while the mixing matrices are allowed to change, i.e.,
\begin{equation}
  \label{eq:spatial_conc}
  \begin{pmatrix}
    \vX_{g_1}\\\vdots\\\vX_{g_m}
  \end{pmatrix}
  =
  \begin{pmatrix}
    A_{g_1}\\\vdots\\A_{g_m}
  \end{pmatrix}
  \cdot\vS.
\end{equation}
Finally, the third model assumes that both the sources and the mixing
remains fixed across groups which implies that for all
$k\in\{1,\dots,m\}$ it holds that
\begin{equation}
  \label{eq:averaging}
  \vX_{g_k}=A\cdot\vS.
\end{equation}
In all three settings the baseline approach to ICA is to simply apply a
classical ICA to the corresponding concatenated or averaged
data, i.e., to apply the algorithm to the temporally/spatially
concatenated data matrices on the left-hand
side of above equations or the average over groups.
These ad-hoc approaches are appealing, since they postulate straightforward
procedures to solving the problem on grouped data and facilitate
interpretability of the resulting estimates. It is these ad-hoc approaches
that are implemented as the default behavior in toolboxes like the widely used
\textsf{eeglab} for EEG analyses \citep{delorme2004eeglab}.

Several procedures have been proposed tailored to specific
applications that extend on these baselines by employing additional
assumptions.  The most prominent such extensions are tensorial methods
that have found popularity in fMRI analysis.  They express the group
index as an additional dimension (the data is thus viewed as a
$\R^{d \times n \times m}$ tensor) and construct an estimate
factorization of the tensor representation.  Many of these procedures
build on the so called PARAFAC (parallel factor analysis) model
\citep{harshman1970}. Recasting the tensor notation, this model is of
the form \eqref{eq:spatial_conc} with $A_{g_k}=A\cdot D_{g_k}$ for all
groups and for diagonal matrices $D_{g_1},\dots,D_{g_m}$.  As can be
seen from this representation, the PARAFAC model allows the mixing
matrices to change across groups while they are constrained to be the
same up to different scaling of the mixing matrix columns
(intuitively, across groups the source dimensions are allowed to
project with different strengths onto the observed signal dimensions).
Given that the matrices $D_{g_1},\dots,D_{g_m}$ are sufficiently
different it is possible to estimate this model uniquely without
further assumptions. However, in the case that some of these diagonal
matrices are equal identifiability is lost. In such cases
\citet{beckmann2005tensorial} suggest to additionally require that the
individual components of the sources be independent. This is
comparable to the case where uncorrelatedness may not be sufficient for the
separation of sources while independence is.

The \coroICA procedure also allows for grouped-data but aims at
inferring a fixed mixing matrix A, i.e., a model as given in
\eqref{eq:time_conc} is considered. In contrast to vanilla
concatenation procedures, our methodology naturally incorporates
changes across groups by allowing and adjusting for different
stationary confounding noise in each group. We argue why this leads to
a more robust procedure and also illustrate this in our simulations
and real data experiments. More generally, our goal is to learn an
unmixing which allows to generalize to new and previously unseen
groups; think for example about learning an unmixing based on several
different training subjects and extending it to new so far unseen
subjects. Such tasks can appear in brain-computer interfacing
applications and can also be of relevance more broadly in feature
learning for classification tasks where classification models are to
be transferred from one group/domain to another.  Since our aim is to
learn a fixed mixing matrix $A$ that is confounding-robust and readily
applicable to new groups, \coroICA cannot naturally be compared to
models that are based on spatial concatenation~\eqref{eq:spatial_conc}
or fixed sources \emph{and} mixings~\eqref{eq:averaging}; these
methods employ fundamentally different assumptions on the model
underlying the data generating process, the crucial difference being
that we allow the sources and their time courses to change between
groups.

\subsection{Our Contribution}

One strength of our methodology is that it explicates a statistical
model that is sensible for data with group structure and can be
estimated efficiently, while being supported by provable
identification results. Furthermore, providing an explicit model with
all required assumptions enables a constructive discussion about the
appropriateness of such modeling decisions in specific application
scenarios. The model itself is based on a notion of invariance against
confounding structures from groups, an idea that is also related to
invariance principles in causality \citep{haavelmo44,peters2016}; see
also Section~\ref{sec:causalinterpretation} for a discussion on the
relation to causality.

We believe that \coroICA is a valuable
contribution to the ICA literature on the following grounds:
\begin{compactitem}
\item We introduce a methodologically sound framework which extends
  ordinary ICA to settings with grouped data and confounding noise.
\item We prove identifiability of the unmixing matrix under mild
  assumptions, importantly, we explicitly allow for time-dependent
  noise thereby lessening the assumptions required by existing noisy
  ICA methods.
\item We provide an easy to implement estimation procedure.
\item We illustrate the usefulness, robustness, applicability, and
  limitations of our newly introduced \coroICA algorithm as well as
  characterize the advantage of \coroICA over existing ICAs: The
  source separation by \coroICA is more stable across groups since it
  explicitly accounts for group-wise stationary confounding.
\item We provide an open-source scikit-learn compatible ready-to-use
  Python implementation available as \coroICA from the Python Package
  Index repository as well as R and Matlab implementations and an
  intuitive audible example which is available at \oururl.
\end{compactitem}

\section{Methodology}\label{sec:methodology}

We consider a general noisy ICA model inspired by ideas employed in
causality research (see Section~\ref{sec:causalinterpretation}). We
argue below that it allows to incorporate group structure and enables
joint inference on multi-group data in a natural way.  For the model
description, let $S_i=(S^1_i,\dots,S^d_i)^\top\in\R^{d\times 1}$ and
$H_i=(H^1_i,\dots,H^d_i)^\top\in\R^{d\times 1}$ be two independent
vector-valued sequences of random variables where
$i\in\{1,\dots, n\}$.  The components $S^1_i,\dots,S^d_i$ are assumed
to be mutually independent for each $i$ while, importantly, we allow
for any weakly stationary noise $H$. Let
$A\in\R^{d\times d}$ be an invertible matrix. The $d$-dimensional data
process $(X_i)_{i\in\{1,\dots,n\}}$ is generated by the following
noisy linear mixing model
\begin{equation}
  \label{eq:mixing}
  X_i = A\cdot S_i+H_i, \quad \text{for all }i\in\{1,\dots,n\}.
\end{equation}
$X$ is a linear combination of source signals $S$ and
confounding variables $H$. In this model, both $S$ and $H$ are
unobserved. One aims at recovering the mixing matrix $A$ as well as
true source signals $S$ from observations of $X$. Without additional assumptions, the
confounding $H$ makes it impossible to identify the mixing matrix $A$.
Even with additional assumptions it remains a difficult task (see
Section~\ref{sec:classical_ica} for an overview of related ICA
models). Given the mixing matrix $A$ it is straightforward to
recover the confounded source signals
$\widetilde{S}_i=S_i+A^{-1}\cdot H_i$.

Throughout this paper, we denote by
$\vX=(X_1,\dots,X_n)\in\R^{d\times n}$ the observed data matrix and
similarly by $\vS$ and $\vH$ the corresponding (unobserved) source and
confounding data matrices.  For a finite data sample generated by this
model we hence have
\begin{equation*}
  \vX = A\cdot \vS+\vH.
\end{equation*}

In order to distinguish between the confounding $H$ and the source
signals $S$ we assume that the two processes are sufficiently
different. This can be achieved by assuming the existence of a group
structure such that the covariance of the confounding $H$ remains
stationary within a group and only changes across groups.
\begin{assumption}[group-wise stationary confounding]
  \label{assumption:group_structure}
  There exists a collection of $m$ disjoint groups
  $\groups=\{g_1,\dots,g_m\}$ with $g_k\subseteq\{1,\dots,n\}$ and
  $\cup_{k=1}^mg_k=\{1,\dots,n\}$ such that for all $g\in\groups$ the
  process $(H_i)_{i\in g}$ is weakly stationary.
\end{assumption}
Under this assumption and given that the source signals change enough
within groups, the mixing matrix $A$ is identifiable (see
Section~\ref{sec:identifiable}). Similar to existing ICA methods discussed in
Section~\ref{sec:ajd}, we propose to estimate the mixing matrix $A$
by jointly diagonalizing empirical estimates of
dependence matrices.
In contrast to existing methods,
we explicitly allow and adjust for the confounding $H$. The
process of finding a matrix $V$ that simultaneously diagonalizes a set
of matrices is known as joint matrix diagonalization and has been
studied extensively \citep[e.g.,][]{ziehe2004, tichavsky2009}. In
Section~\ref{sec:estimation}, we show how to construct an estimator
for $V$ based on approximate joint matrix diagonalization.

The key step in adjusting for the confounding is to make use of the
assumption that in contrast to the signals $S$ the confounding $H$
remains stationary within groups. Depending on the type of signal in
the sources one can consider different sets of matrices. Here, we
distinguish between two types of signals.

\paragraph{Variance signal}
In case of a variance signal, the variance process of each signal
source $\var(S^j_i)$ changes over time. These changes can
be detected by examining the covariance matrix $\cov(X_i)$ over
time. For $V=A^{-1}$ and using \eqref{eq:mixing} it holds for all
$i\in\{1,\dots,n\}$ that
\begin{equation*}
  V\cov(X_i)V^{\top}=\cov(S_i)+V\cov(H_i)V^{\top}.
\end{equation*}
Since the source signal components $S_i^j$ are mutually independent,
the covariance matrix $\cov(S_i)$ is diagonal. Moreover, due to
Assumption~\ref{assumption:group_structure} the covariance matrix of
the confounding $H$ is constant, though not necessarily diagonal, within each group.
This implies for
all groups $g\in\groups$ and for all $k,l\in g$ that
\begin{equation}
  \label{eq:jointdiag_a}
  V\left(\cov(X_k)-\cov(X_l)\right)V^{\top}=\cov(S_k)-\cov(S_l)
\end{equation}
is a diagonal matrix.

\paragraph{Time-dependence signal}
In case of a time-dependence signal, the time-dependence of each
signal source $S^j_i$ changes over time, i.e., for fixed $\tau$,
$\cov(S_i^j,S_{i-\tau}^j)$ changes over time. These changes lead to
changes in the auto-covariance matrices
$\cov(X_i,X_{i-\tau})$. Analogous to the variance signal it holds for
all $i\in\{\tau+1,\dots,n\}$ that
\begin{equation*}
  V\cov(X_i,X_{i-\tau})V^{\top}=\cov(S_i,S_{i-\tau})+V\cov(H_i,H_{i-\tau})V^{\top}.
\end{equation*}
Since the source signal components $S_i^j$ are mutually
independent, the auto-covariance matrix $\cov(S_i,S_{i-\tau})$ is
diagonal and due the stationarity of $H$ (see
Assumption~\ref{assumption:group_structure}) the auto-covariance
$\cov(H_i,H_{i-\tau})$ is constant within each group. This implies for
all groups $g\in\groups$, for all $k,l\in g$ and for all $\tau$
that
\begin{equation}
  \label{eq:jointdiag_b}
  V\left(\cov(X_k,X_{k-\tau})-\cov(X_l,X_{l-\tau})\right)V^{\top}=\cov(S_k,S_{k-\tau})-\cov(S_l,S_{l-\tau})
\end{equation}
is a diagonal matrix.

\-\\
For both signal types, we can identify $V$ by simultaneously
diagonalizing differences of \mbox{(auto-)covariance}
matrices. Details and identifiability results are given in
Section~\ref{sec:estimation}. The two signal types considered differ
from both, the more classical settings of non-Gaussian
time-independent signals as considered for example by fastICA, and the
stationary signals with fixed time-dependence assumed for SOBI (cf.\
Table~\ref{table:ica_methods}). Owing to the non-stationarity of the
signal we can allow for more general forms of noise.

\subsection{Motivating Examples}

To get a better understanding of our proposed ICA model in
\eqref{eq:mixing}, we illustrate two different aspects: the
group structure and the noise model.

\paragraph{Noise model} \coroICA can be viewed as a noisy ICA, where the
noise is allowed to be group-wise non-stationary. This generalizes
existing noisy ICA methods, which, to the best of our knowledge, all
assume that the noise is independent over time with various additional
assumptions. The following example illustrates the intuition behind
our model via a toy-application to natural images.
\begin{example}[unmixing noisy images]\label{ex:image}
  We provide an illustration of how our proposed method compares to
  other ICA approaches under the presence of noise.  Four images, each
  $450\times 300$ pixels and with three RGB color channels, are used
  to construct four sources $S^1,S^2,S^3,S^4$ as follows.\footnote{The
    images are freely available from \citet{images}.}  Every color
  channel is converted to a one dimensional vector by cutting each
  image into $15\times 10$ equally sized patches (i.e., each patch
  consists of $30\times 30$ pixels) and concatenating the row-wise
  vectorized patches.  This procedure preserves the local structure of
  the image.  We concatenate the three color channels and
  consider them as separate groups for our model. Thus, each of the four
  sources $S^1,\dots,S^4$ consists of
  $n=3\cdot 450\cdot 300=405.000$ observations, that is, three groups of
  $135.000$ observations corresponding to the RGB color
  channels. Next, we construct locally dependent noise that differs
  across color channels. Here, locally dependent means that the added
  noise is similar (and dependent) for pixels which are close to each
  other. This results in four noise processes $H^1,\dots,H^4$. We
  combine the sources with the noise and apply a random mixing matrix
  $A$ to obtain the following observed data
  \begin{equation*}
    X=A\cdot S+H.
  \end{equation*}
  The recast noisy images $\widetilde{S}=S+A^{-1} H$ are illustrated
  in the first row and the recast observed mixtures $X$ in the second
  row of Figure~\ref{fig:image}. The last three rows are the resulting
  reconstructions of three different ICA procedures, \coroICATD, fastICA
  and choiICA\,(TD).  As expected, fastICA as a noise-free ICA method,
  appears frail to the noise in the images. While choiICA\,(TD) is able
  to adjust for independent noise, it is unable to properly adjust for
  the spatial dependence of the noise process and thus leads to
  undesired reconstruction results. In contrast, \coroICATD is able to
  recover the noisy images. It is the noise and its
  characteristics that break the two competing ICA methods, since all
  three methods are able to unmix the images in the noise-free case
  (not shown here).
\end{example}
\begin{figure}[t]
  \centering
  \includegraphics{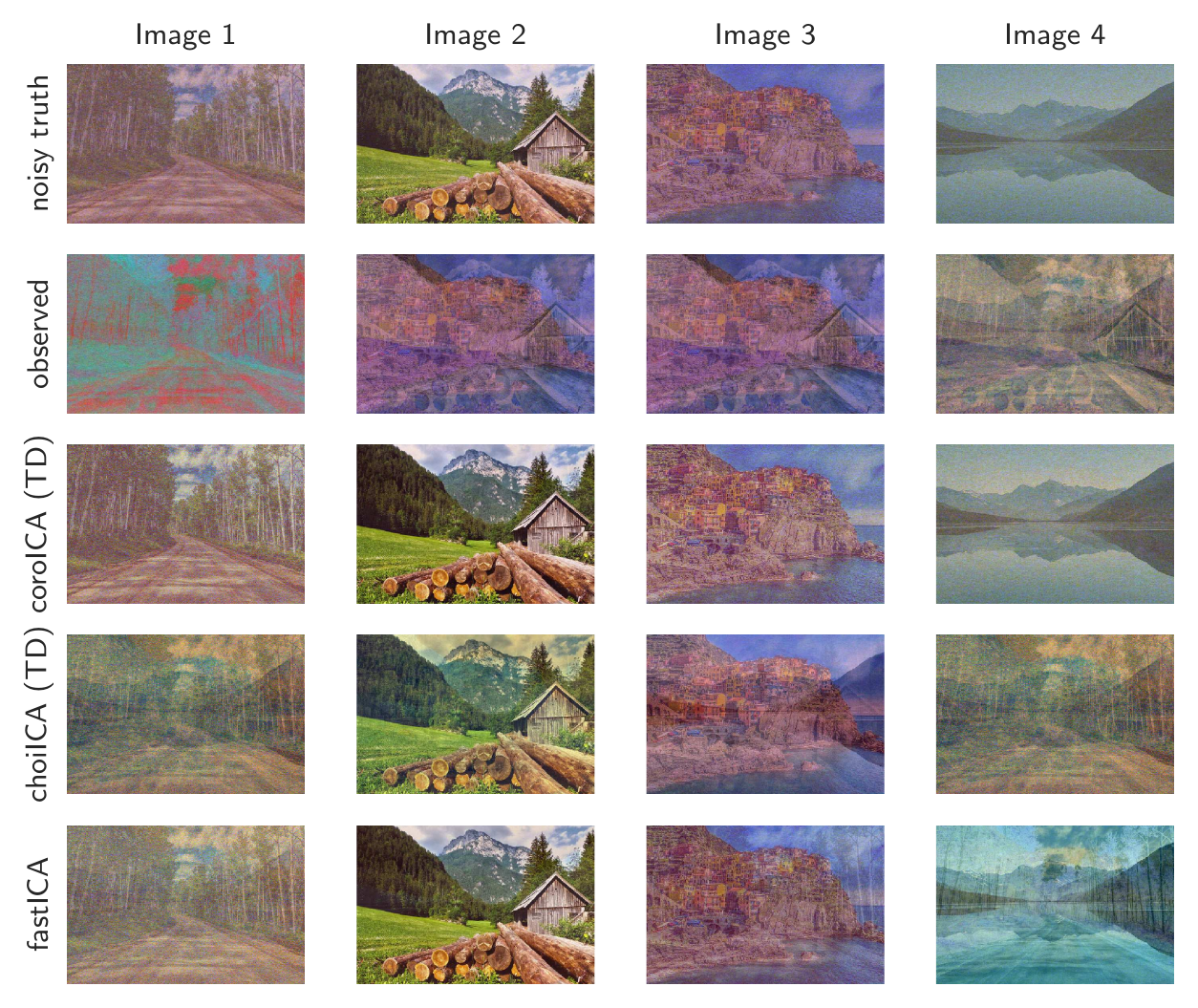}
  \caption{Images accompanying example~\ref{ex:image}. The top row shows
    noisy unmixed images, the second row shows mixed images, and the last three
    rows show unmixed and rescaled images resulting from an
    application of \coroICATD, choiICA\,(TD) and fastICA (cf.\ Table~\ref{table:ica_methods}).
    Here, only \coroICATD is able to correctly unmix the images and recover the original (noise-corrupted) images.}
  \label{fig:image}
\end{figure}
The noise model we employ is motivated by recent advances in causality
research where the group-wise stationary noise can be interpreted as
unobserved confounding factors in linear causal feedback models. We
describe this in more detail with an explicit example application to
Antarctic ice core data in Section~\ref{sec:causalinterpretation}.

\paragraph{Group structure}
A key aspect of our model is that it aims to leverage group-structure to improve
the stability of the umixing under the presence of group-wise confounding.
Here we refer to the following notion of stability:
A stable unmixing matrix extracts the same set of independent sources
when applied to the different groups; it is robust against the
confounding that varies across groups and introduces dependences.
A standard ICA method is not able to estimate the correct unmixing $V=A^{-1}$,
if the data generating process follows
our confounded ICA model in \eqref{eq:mixing}.
These methods extract signals that are not only corrupted by the group-wise
confounding but also are mixtures of the independent sources
and are thus not stable in the aforementioned sense. This
is illustrated by the ``America's Got Talent Duet Problem'' (cf.\
Example~\ref{ex:toy_example}), an extension and alteration of the
classical ``cocktail party problem''.

\begin{example}[America's Got Talent Duet Problem]
  \label{ex:toy_example}
  Consider the problem of evaluating two singers at a duet audition
  individually. This requires to listen to the two voices
  separately, while the singers perform simultaneously.  There are two
  sound sources in the audition room (the two singers) and
  additionally several noise sources which corrupt the recordings at
  the two microphones (or the jury member's two ears). A schematic of
  such a setting is illustrated in Supplement~\ref{sec:complementary},
  Figure~\ref{fig:crowdedroom}. The additional noise comes from an
  audience and two open windows. One can assume that this noise
  satisfies our Assumption~\ref{assumption:group_structure} on a
  single group. The sound stemming from the audience can be seen as an
  average of many sounds, hence remaining approximately
  stationary over time. Typical sounds from an open window also
  satisfy this assumption, for example sound from a river or
  a busy road.  Our methodology, however, also allows for more
  complicated settings in which the noise shifts at known points in
  times, for example if someone opens or closes a window or starts
  mowing the lawn outside. In such cases we use the known time blocks
  of stationary noise as groups and apply \coroICAvar on this grouped
  data. An example with artificial sound data related to this setting
  is available at \oururl. We show that \coroICAvar is able to recover
  useful sound signals with the two voices being separated into
  different dimensions and thus allows to listen to them individually.
  In contrast, existing ICAs applied to the time concatenated data fail to
  unmix the two singers.
\end{example}

\subsection{Identifiability}\label{sec:identifiable}

Identifiability requires that the source signals $S$ change
sufficiently strong within groups. The precise notion of a strong
signal depends on the type of signal. As discussed previously, we
consider two types of non-stationary signals (i) variance signals and
(ii) time-dependence signals. Depending on the signal type we
formalize two slightly different assumptions that characterize
source signals that ensure identifiability.
Firstly, in the case of a variance signal, we have the following
assumption.
\begin{assumption}[signals with independently changing variance]
  \label{assumption:var}
  For each pair of components $p, q\in\{1,\dots,d\}$ we require the
  existence of three (not necessarily unique) groups
  $g_1,g_2,g_3\in\groups$ and three corresponding pairs
  $l_1,k_1\in g_{1}$, $l_2,k_2\in g_{2}$ and \mbox{$l_3,k_3\in g_{3}$} such
  that the two vectors
  {\renewcommand*{\arraystretch}{1.3}
    \begin{equation*}
      \begin{pmatrix}
        \var\big(S^{p}_{l_1}\big)-\var\big(S^{p}_{k_1}\big) \\
        \var\big(S^{p}_{l_2}\big)-\var\big(S^{p}_{k_2}\big) \\
        \var\big(S^{p}_{l_3}\big)-\var\big(S^{p}_{k_3}\big)
      \end{pmatrix}
      \text{ and }
      \begin{pmatrix}
        \var\big(S^{q}_{l_1}\big)-\var\big(S^{q}_{k_1}\big) \\
        \var\big(S^{q}_{l_2}\big)-\var\big(S^{q}_{k_2}\big) \\
        \var\big(S^{q}_{l_3}\big)-\var\big(S^{q}_{k_3}\big)
      \end{pmatrix}
    \end{equation*}}%
  are neither collinear nor equal to zero.
\end{assumption}
In case of time-dependence signals we have the analogous assumption.
\begin{assumption}[signals with independently changing time-dependence]
  \label{assumption:time-dependence}
  For each pair of components $p, q\in\{1,\dots,d\}$ we require the
  existence of three (not necessarily unique) groups
  $g_1,g_2,g_3\in\groups$ and three corresponding pairs
  $l_1,k_1\in g_{1}$, $l_2,k_2\in g_{2}$ and \mbox{$l_3,k_3\in g_{3}$} for
  which there exists $\tau\in\{1,\dots,n\}$ such that the two vectors
  {\renewcommand*{\arraystretch}{1.3}
    \begin{equation*}
      \begin{pmatrix}
        \cov\big(S^{p}_{l_1},S^{p}_{l_1-\tau}\big)-\cov\big(S^{p}_{k_1},S^{p}_{k_1-\tau}\big) \\
        \cov\big(S^{p}_{l_2},S^{p}_{l_2-\tau}\big)-\cov\big(S^{p}_{k_2},S^{p}_{k_2-\tau}\big) \\
        \cov\big(S^{p}_{l_3},S^{p}_{l_3-\tau}\big)-\cov\big(S^{p}_{k_3},S^{p}_{k_3-\tau}\big)
      \end{pmatrix}
      \text{ and }
      \begin{pmatrix}
        \cov\big(S^{q}_{l_1},S^{q}_{l_1-\tau}\big)-\cov\big(S^{q}_{k_1},S^{q}_{k_1-\tau}\big) \\
        \cov\big(S^{q}_{l_2},S^{q}_{l_2-\tau}\big)-\cov\big(S^{q}_{k_2},S^{q}_{k_2-\tau}\big) \\
        \cov\big(S^{q}_{l_3},S^{q}_{l_3-\tau}\big)-\cov\big(S^{q}_{k_3},S^{q}_{k_3-\tau}\big)
      \end{pmatrix}
    \end{equation*}}%
  are neither collinear nor equal to zero.
\end{assumption}
Intuitively, these assumptions ensure that the signals are not
changing in exact synchrony across components, which removes
degenerate types of signals. In particular, they are satisfied in the
case that the variance or auto-covariance processes change pair-wise
independently over time. Whenever one of these assumptions is
satisfied, the mixing matrix $A$ is uniquely identifiable.
\begin{theorem}[identifiability of the mixing matrix]\label{thm:mixing}\-\\
  Assume the data process $(X_i)_{i\in\{1,\dots,n\}}$ satisfies the
  model in \eqref{eq:mixing} and
  Assumption~\ref{assumption:group_structure} holds. If additionally
  either Assumption~\ref{assumption:var} or
  Assumption~\ref{assumption:time-dependence} is satisfied, then $A$
  is unique up to permutation and rescaling of its columns.
\end{theorem}
\begin{proof}
A proof is given in Supplement~\ref{sec:proofs}.
\end{proof}

\subsection{Estimation}\label{sec:estimation}

In order to estimate $V$ from a finite observed sample
$\vX \in \R^{d \times n}$, we first partition each group into
subgroups. We then compute the empirical (auto-)covariance matrices on
each subgroup. Finally, we estimate a matrix that simultaneously
diagonalizes the differences of these empirical (auto-)covariance
matrices using an approximate joint matrix diagonalization
technique. This procedure results in three methods depending on which
type of matrices we diagonalize. Similar to our notation for the
different versions of choiICAs we denote these methods by
\coroICAvar if we diagonalize differences of covariances,
\coroICATD if we diagonalize differences of auto-covariances, and
\coroICAvarTD if we diagonalize both differences of covariance
and auto-covariances.

More precisely, for each group $g\in\groups$, we first
construct a partition $\mathcal{P}_g$ consisting of subsets of $g$
such that each $e\in\mathcal{P}_g$ satisfies that $e\subseteq g$ and
$\cup_{e\in\mathcal{P}_g}e=g$. This partition $\mathcal{P}_g$ should
be granular enough to capture the changes in the signals described in
Assumption~\ref{assumption:var} or~\ref{assumption:time-dependence}.
We propose partitioning each group based on a grid such that the
separation between grid points is large enough for a reasonable
estimation of the covariance matrix and at the same time small enough
to capture variations in the signals. In our experiments, we observed
robustness with respect to the exact choice; only too small partitions
should be avoided since otherwise the procedure is fragile due to
poorly estimated covariance matrices. More details on the choice of
the partition size are given in Remark~\ref{rmk:partition}. Depending
on whether a variance or time-dependence signal or a hybrid thereof is
considered, we fix time lags $\Tau \subset \N_0$.

Next, for each group $g\in\groups$, each distinct pair
$e,f\in\mathcal{P}_g$, and each $\tau \in \Tau$ we define the matrix
\begin{equation*}
  M^{g,\tau}_{e,f}\coloneqq \empcov_\tau(\vX_e)-\empcov_\tau(\vX_f),
\end{equation*}
where $\empcov_\tau(\cdot)$ denotes the empirical (auto-)covariance
matrix for lag $\tau$ and $\vX_e$ is the data matrix restricted to the
columns corresponding to the subgroup
$e$. Assumption~\ref{assumption:group_structure} ensures that
$V M^{g,\tau}_{e,f} V^{\top}$ is approximately diagonal. We are
therefore interested in finding an invertible matrix $V$ which
approximately jointly diagonalizes the matrices in the set
\begin{equation}
  \label{eq:comparison_set_all}
  \mathcal{M}^{\operatorname{all}}
  \coloneqq\big\{M^{g,\tau}_{e,f}
  \,\big\rvert\, g\in\groups
  \text{ and } e,f\in\mathcal{P}_g \text{ and } \tau\in\Tau \big\}.
\end{equation}
The number of matrices in this set grows quadratically in the number of
partitions. This can lead to large numbers of
matrices to be diagonalized. Another option that reduces the computational load is
to compare each partition to its complement, which
leads to the following set of matrices
\begin{equation}
  \label{eq:comparison_set_comp}
  \mathcal{M}^{\operatorname{comp}}
  \coloneqq\big\{M^{g,\tau}_{e, \bar{e}} \,\big\rvert\, g\in\groups
  \text{ and } e\in\mathcal{P}_g \,(\text{with }\bar{e}\coloneqq g\setminus e) \text{ and } \tau\in\Tau\big\}
\end{equation}
or to compare only neighboring partitions as in
\begin{equation}
  \label{eq:comparison_set_neighbor}
  \mathcal{M}^{\operatorname{neighbor}}
  \coloneqq\big\{M^{g,\tau}_{e,\operatorname{neighbor}(e)} \,\big\rvert\, g\in\groups
  \text{ and } e\in\mathcal{P}_g \text{ and } \tau\in\Tau\big\},
\end{equation}
where $\operatorname{neighbor}(e)$ is the partition to the right of
$e$.

The task of jointly diagonalizing a set of matrices is a
well-studied topic in the literature and is referred to as approximate
joint matrix diagonalization.  Many solutions have been proposed for
different assumptions made on the matrices to be diagonalized. In this
paper, we use the \textsf{uwedge} algorithm\footnote{As a byproduct of
  our work, we are able to provide a new stable open-source
  Python/R/Matlab implementation of the \textsf{uwedge} algorithm
  which is also included in our respective \coroICA packages.}
introduced by \citet{tichavsky2009}. The basic idea behind
\textsf{uwedge} is to find a minimizer of a proxy for the
loss function
\begin{equation*}
  \ell(V)=\sum_{M\in\mathcal{M}^{*}}\left(\sum_{k\neq l}\big[V M V^{\top}\big]_{k,l}^2\right),
\end{equation*}
over the set of invertible matrices, where in our case
$\mathcal{M}^* \in \{\mathcal{M}^{\operatorname{all}},
\mathcal{M}^{\operatorname{comp}},\mathcal{M}^{\operatorname{neighboring}}\}$.

The full estimation procedure based on the set
$\mathcal{M}^{\operatorname{neighbouring}}$ defined in \eqref{eq:comparison_set_comp}
is made explicit in the pseudo code in Algorithm~\ref{alg:coroICA}
(where ApproximateJointDiagonalizer stands for a general approximate joint diagonalizer;
here we use \textsf{uwedge}).

\begin{remark}[choosing the partition and the lags]
  \label{rmk:partition}
  Whenever there is no obvious partition of the data, we propose to
  partition the data into equally sized blocks with a fixed partition
  size. The decision on how to choose a partition size should be
  driven by type of non-stationary signal one expects and the
  dimensionality of the data. For example, in the case of a variance signal
  the partition should be fine enough to capture areas of high
  and low variance, while at the same time being coarse enough to
  allow for sufficiently good estimates of the covariance matrices. That
  said, for applications to real data sets the signals are often of
  various length implying that there is a whole range of partition
  sizes which all work well. In cases with few data points, it can
  then be useful to consider several grids with different partition
  sizes and diagonalize across all resulting differences
  simultaneously. This somewhat removes the dependence of the results
  on the exact choice of a partition size and increases the power of
  the procedure.  We employ this approach in Section~\ref{ex:climate}.
  In general, the lags $\Tau$ should be chosen as $\Tau = \{0\}$,
  $\Tau \subset \N$, or $\Tau \subset \N_0$, depending on whether a
  variance signal, time-dependence signal, or a hybrid thereof is
  considered.  For time-dependence signal, we recommend to determine
  up to which time-lag the autocorrelation of the observed signals has
  sufficiently decayed, and use all lags up to that point.
\end{remark}

\begin{algorithm}[h]
\SetAlgoLined
\SetKwInOut{Input}{input}
\SetKwInOut{Output}{output}
\Input{data matrix $\vX$\newline
 group index $\groups$ (user selected)\newline
 group-wise partition $(\mathcal{P}_g)_{g\in\groups}$ (user selected)\newline
 lags $\Tau \subset \N_0$ (user selected)
 }

 initialize empty list $\mathcal{M}$

\For{$g\in\groups$}{
  \For{$e\in\mathcal{P}_g$}{
    \For{$\tau\in\Tau$}{
      append $\empcov_\tau(\vX_e)-\empcov_\tau(\vX_{\operatorname{neighbour}(e)})$ to list $\mathcal{M}$
    }
  }
 }
 $\widehat{V}\gets\operatorname{ApproximateJointDiagonalizer}(\mathcal{M})$

 $\widehat{\vS}\gets \widehat{V}\vX$

 \Output{unmixing matrix $\widehat{V}$\newline
         sources $\widehat{\vS}$}
 \caption{\coroICA}
 \label{alg:coroICA}
\end{algorithm}

\subsection{Assessing the Quality of Recovered Sources}\label{sec:scores}

Assessing the quality of the recovered sources in an ICA setting is an
inherently difficult task, as is typical for unsupervised learning
procedures. The unidentifiable scale and ordering of
the sources as well as the unclear choice of a performance measure
render this task difficult. Provided that ground truth is known,
several scores have been proposed, most notably the
Amari measure introduced by \citet{amari1996} and the minimum distance
(MD) index due to \citet{ilmonen2010}. Here, we use the MD
index, which is defined as
\begin{equation*}
  \operatorname{MD}(\hat{V}, A) = \frac{1}{\sqrt{p-1}}\inf_{C\in\mathcal{C}}\norm{C\hat{V}A-\identity},
\end{equation*}
where the set $\mathcal{C}$ consists of matrices for which each row
and column has exactly one nonzero element. Intuitively, this score
measures how close $\hat{V}A$ is to a rescaled and permuted version of
the identity matrix. One appealing property of this score is that it
can be computed efficiently by solving a linear sum assignment problem. In
contrast to the Amari measure, the MD index is affine invariant and
has desirable theoretical properties \citep[see][]{ilmonen2010}.

We require a different performance measure for our real data
experiments where the true unmixing matrix is unknown. Here, we check
whether the desired independence (after adjustment for the constant
confounding) is achieved by computing the following covariance
instability score (CIS) matrix. It measures the instability of the
covariance structure of the unmixed sources $\widehat{\vS}$ and is
defined for a each groups $g\in\groups$ and a corresponding partition
$\mathcal{P}_g$ (see Section~\ref{sec:estimation}) by
\begin{equation*}
  \operatorname{CIS}(\widehat{\vS},\mathcal{P}_g)\coloneqq
  \frac{1}{\abs{\mathcal{P}_g}}
  \sum_{e\in\mathcal{P}_g}\left(\frac{\empcov(\widehat{\vS}_e)-\empcov(\widehat{\vS}_{\operatorname{neighbour}(e)})}{\widehat{\sigma}_{\widehat{\vS}_g}\cdot\widehat{\sigma}_{\widehat{\vS}_g}^{\top}}\right)^2,
\end{equation*}
where $\widehat{\sigma}_{\widehat{\vS}}\in\R^{d\times 1}$ is the
empirical standard deviation of $\widehat{\vS}$ and the fraction is
taken element-wise. The CIS matrix is approximately diagonal whenever
$\widehat{\vS}$ can be written as the sum of independent source
signals $\vS$ and confounding $\vH$ with fixed covariance.  This is
condensed into one scalar that reflects how stable the sources'
covariance structure is by averaging the off-diagonals of the CIS
matrix
\begin{equation*}
  \operatorname{MCIS}(\widehat{\vS}, \mathcal{P}_g)^2\coloneqq
  \frac{1}{d(d-1)}\sum_{\substack{i,j=1\\i\neq j}}^d\left[\operatorname{CIS}(\widehat{\vS}, \mathcal{P}_g)\right]_{i,j}.
\end{equation*}
The differences taken in the CIS score extract the variance signals
such that the mean covariance instability score (MCIS) can be
understood as a measure of independence between the recovered variance signal processes.
High values of MCIS imply strong dependences beyond stationary
confounding between the signals. Low values imply weak
dependences. MCIS is a reasonable score whenever there is a \emph{variance
  signal} (as described in Section~\ref{sec:methodology}) in sources
and is a sensible evaluation metric of ICA procedures in such cases.
In case of time-dependence signal (as described in Section~\ref{sec:methodology}),
one can define an analogous score based on the auto-covariances.
Here, we restrict ourselves to the variance signal case as for all our
applications this appeared to constitute the dominant part of the signal.

In case of \emph{variance signals} the MCIS appears natural and appropriate
as independence measure: It measures
how well the individual variance signals (and hence the relevant information)
are separated. To get a better intuition, let
$A=(a_1,\dots,a_d)\in\R^{d\times d}$ denote the mixing and
$V=(v_1,\dots,v_d)^\top\in\R^{d\times d}$ the corresponding unmixing
matrix (i.e., $V=A^{-1}$, $a_i$ are columns of $A$ and $v_i$ are rows
of $V$).  Then it holds that,
\begin{align}
    \label{eq:covXv}
    \nonumber
  \cov(X_i)v_j^{\top}&=A\cov(S_i)A^{\top}v_j^{\top}+\cov(H_i)v_j^{\top}\\
    \nonumber
  &=A\cov(S_i)e_j^{\top}+A\cov(H_i)e_j^{\top}\\
  &=a_j\var(S_i^j)+A\cov(H_i)e_j^{\top}
\end{align}
Under our group-wise stationary confounding assumption
(Assumption~\ref{assumption:group_structure}) this implies that within
all groups $g\in\groups$, it holds for all $l,k\in g$ that
\begin{equation}
  \label{eq:stable_comp2}
  \left(\cov(X_l)-\cov(X_k)\right)v_j^{\top}
  =a_j\left(\var(S^j_l)-\var(S^j_k)\right).
\end{equation}
This equation holds also in the confounding-free case and it reflects the
contribution of the signal (in terms of variance signal) of the $j$-th recovered
source $S^j$ to the the variance signal in all components of the
observed multivariate data $X$.

While in the population case the equality in \eqref{eq:stable_comp2}
is satisfied exactly, this is no longer the case when the (un-)mixing
matrix is estimated on finite data.  Consider
two subsets $e, f\in g$ for some group $g\in\groups$, then using the
notation from Section~\ref{sec:estimation} and denoting by
$\widehat{v}_j$ and $\widehat{a}_j$ the estimates of $v_j$ and $a_j$,
respectively, it holds that
\begin{align}
  \label{eq:empcovXv}
  \nonumber
  M_{e,f}^g\widehat{v}_j^{\top}
  &=\big[\empcov(\vX_e)-\empcov(\vX_f)\big]\widehat{v}_j^{\top}\\
  \nonumber
  &=\widehat{A}\big[\empcov(\widehat{S}_e)-\empcov(\widehat{S}_f)\big]\widehat{A}^{\top}\widehat{v}_j^{\top}\\
  &=\widehat{A}\big[\empcov(\widehat{S}_e)-\empcov(\widehat{S}_f)\big]e_j^{\top}\nonumber\\
  &\approx \widehat{a}_j(\var(S_{e}^j)-\var(S_{f}^j)).
\end{align}
The approximation is close only if the empirical estimate
$\widehat{V}$ correctly unmixes the $j$-th
source. Essentially, MCIS measures the extent to which this
approximation holds true for all components simultaneously across the subsets
specified by the partition $\mathcal{P}_g$. It is also
possible to consider individual components by assessing how
closely the following proportionality is satisfied
\begin{equation}
  \label{eq:component_instability}
  \sum_{M\in\mathcal{M}^*}\operatorname{sign}(\widehat{v}_jM\widehat{v}_j^{\top}) M\widehat{v}_j^{\top}\propto \widehat{a}_j.
\end{equation}\label{eq:activation_map}
In EEG experiments,
this can also be assessed visually by comparing the topographic
maps corresponding to columns of A with so-called activation maps corresponding to the
left-hand side in \eqref{eq:component_instability}. More details on
this are provided in Section~\ref{sec:topographic_maps}.

\FloatBarrier

\section{Causal Perspective}\label{sec:causalinterpretation}

Our underlying noisy ICA model \eqref{eq:mixing} and the assumption on
the noise (Assumption~\ref{assumption:group_structure}) are motivated
by causal structure learning scenarios. ICA is closely linked to the problem of identifying
structural causal models (SCMs) \citep[see][]{Pearl2009, Imbens2015,
  Peters2017}. \citet{shimizu2006} were the first to make this
connection explicit and used ICA to infer causal structures. To make this more
precise consider the following linear SCM
\begin{equation}
  \label{eq:linearSCM}
  X_i=B\cdot X_i +\widetilde{S}_i,
\end{equation}
where $X_i$ are observed covariates and $\widetilde{S}_i$ are noise
terms. An SCM induces a corresponding causal graph over the involved
variables by drawing an edge from variables on the right-hand side to
the one on the left-hand side of \eqref{eq:linearSCM}. Moreover, we can define noise
interventions \citep{Pearl2009} by allowing the distributions of the
noise terms $\widetilde{S}_i$ to change for different $i$. In the
language of ICA, this means that the signals $\widetilde{S}_i$ encode
the different interventions (over time) on the noise
variables. Assuming that the matrix $\identity-B$ is invertible, we
can rewrite \eqref{eq:linearSCM} as
\begin{equation*}
  X_i=(\identity-B)^{-1} \widetilde{S}_i,
\end{equation*}
which can be viewed as an ICA model with mixing matrix
$A=(\identity-B)^{-1}$. Instead of taking the noise term
$\widetilde{S}_i$ as independent noise sources one can also consider
$\widetilde{S}_i=S_i+H_i$. In that case the linear SCM in
\eqref{eq:linearSCM} describes a causal model between the observed
variables $X_i$ in which hidden confounding is allowed. This is
illustrated in Figure~\ref{fig:SCMillustration}, which depicts a 3
variable SCM with feedback loops and confounding. Learning a causal
model as in \eqref{eq:linearSCM} with ICA is generally done by
performing the following two steps.
\begin{enumerate}[(i)]
\item \textbf{(ICA)} The matrix $(\identity-B)$ is inferred by ICA up to an
  undefined scale and permutation of its rows by using an appropriate
  ICA procedure. This step is often infeasible in the presence of
  confounding $H$ since existing ICA methods only allow noise
  under restrictive assumptions (cf. Table~\ref{table:ica_methods}).
\item \textbf{(identify $\mathbf{B}$)} There are essentially two
  assumptions that one can make in order for this to work. The first
  is to assume the underlying causal model has an acyclic structure as in
  \citet{shimizu2006}. In such cases the matrix $B$ needs to be permuted
  to an upper triangular matrix. The second option is to
  allow for feedback loops in the causal model but restrict the types
  of feedback to exclude infinite loops as in
  \citet{hoyer2008} and \citet{rothenhausler2015}.
\end{enumerate}
When performing step (i) there are two important modeling assumptions
that are made when selecting the ICA procedure: (a) the type of
allowed signals (types of interventions) and (b) the type of allowed
confounding. For the classic ICA setting with non-Gaussian source
signals and no noise this translates to the class of linear
non-Gaussian models, such as Linear Non-Gaussian Acyclic Models
(LiNGAMs) introduced by \citet{shimizu2006}. While such models are a
sensible choice in a purely observational setting (i.e., no samples
from interventional settings) they are somewhat misspecified in terms
of (a) when data from different interventional settings or
time-continuous intervention shifts are observed (see
Remark~\ref{rmk:intervention}). In those settings, it is more natural
to use ICA methods that are tailored to sequential shifts as for
example choiICA or \coroICA. Moreover, most common ICA methods
consider noise-free mixing, which from a causal perspective implies
that no hidden confounding is allowed. While noisy ICA weakens this
assumption, existing methods only allow for time-independent or even
iid noise, which again greatly restricts the type of confounding. In
contrast, our proposed \coroICA allows for any type of block-wise
stationary confounding, hence greatly increasing the class of causal
models which can be inferred. This is attractive for causal modeling
as it is a priori unknown whether hidden confounding
exists. Therefore, our proposed procedure allows for robust causal
inference under general confounding settings. In
Section~\ref{ex:climate}, we illustrate a potential application to
climate science and how the choice of ICA can have a strong impact on
the estimates of the causal parameters.

\begin{remark}[relation between interventions and non-stationarity]
  \label{rmk:intervention}
  A causal model does not only describe the observational distribution
  but also the behavior of the data generating model under all of the
  allowed interventions. Here, we restrict the allowed interventions
  to distribution shifts in the source signals, that either change the
  distribution block-wise (e.g., abruptly changing environmental
  conditions) or continuously (e.g., continuous shifts in the
  environmental conditions). Any such shifts are by definition
  synonymous with the process $S_i$ being non-stationary. In our
  proposed causal model \eqref{eq:linearSCM} the non-stationarity of the
  signal therefore corresponds to shifts in the environmental
  conditions which can be utilized, using \coroICA, to
  infer the underlying causal structure. From this perspective, the causal
  inference procedure we propose here is a method based on
  interventional data rather than plainly observational data, while the interventions
  are not exactly known.
\end{remark}

\begin{figure}[t]
    \centering
    \hfill
\begin{minipage}{0.5\textwidth}
  \scalebox{0.6}{
    \begin{tikzpicture}[framed, background
      rectangle/.style={thick, draw=black,
        rounded corners}, scale=5]
      \tikzstyle{VertexStyle} = [shape = circle, minimum width = 3em,draw]
      \SetGraphUnit{2}
      \Vertex[Math,L=X^1,x=-0.6,y=0]{X1}
      \Vertex[Math,L=X^2,x=0.6,y=0]{X2}
      \Vertex[Math,L=X^3,x=0,y=0.6]{X3}
      \Vertex[Math,L=S^1,x=-1,y=0]{S1}
      \Vertex[Math,L=S^2,x=1,y=0]{S2}
      \Vertex[Math,L=S^3,x=0,y=1]{S3}
      \tikzstyle{VertexStyle} = [shape = circle, minimum width =
      3em,draw,color=colA]
      \Vertex[Math,L=H^1,x=-0.6,y=0.6]{H1}
      \Vertex[Math,L=H^2,x=0.6,y=0.6]{H2}
      \tikzstyle{EdgeStyle} = [->,>=stealth',shorten > = 2pt]
      \Edge(S1)(X1)
      \Edge(S2)(X2)
      \Edge(S3)(X3)
      \Edge(X1)(X3)
      \Edge(X3)(X2)
      \tikzset{EdgeStyle/.append style = {->, bend left}}
      \Edge(X1)(X2)
      \Edge(X2)(X1)
      \tikzstyle{EdgeStyle} = [->,>=stealth',shorten > = 2pt, color=colA]
      \Edge(H1)(X1)
      \Edge(H1)(X3)
      \Edge(H2)(X2)
      \Edge(H2)(X3)
    \end{tikzpicture}
  }
\end{minipage}%
\hfill
\begin{minipage}{0.5\textwidth}
  \begin{mdframed}[roundcorner=5pt]
    \begin{align*}
      X^1 &\leftarrow b_{1,2}X^2 &&+ S^1+\textcolor{colA}{H^1}\\
      X^2 &\leftarrow b_{2,1}X^1 + b_{2,3}X^3 &&+ S^2+\textcolor{colA}{H^2}\\
      X^3 &\leftarrow b_{3,1}X^1 &&+ S^3+\textcolor{colA}{H^1}+\textcolor{colA}{H^2}
    \end{align*}
    \phantom{1}
  \end{mdframed}
\end{minipage}
\hfill
    \caption{Illustration of an SCM with (including colored nodes \textcolor{colA}{$H^1$}, \textcolor{colA}{$H^2$})  and
  without (excluding colored nodes) confounding.}
\label{fig:SCMillustration}
\end{figure}
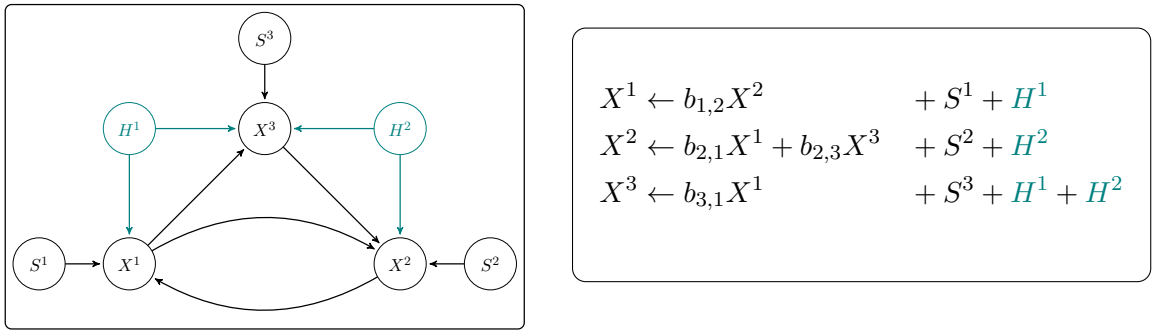

\subsection{Application to Climate Science}\label{ex:climate}

To motivate the foregoing causal model we consider a prominent
example from climate science: the causal relationship between carbon
dioxide concentration ($\text{CO}_2$) and temperature (T). More precisely, we
consider Antarctic ice core data that consists of
temperature and carbon dioxide measurements of the past 800'000
years due to \citet[carbon dioxide]{Bereiter2014} and
\citet[temperature]{jouzel2007orbital}. We combined both temperature
and carbon dioxide data and recorded measurements every 500 years by
a cubic interpolation of the raw data. The data is shown in
Figure~\ref{fig:climate_results} (right). Oversimplifying,
one can model this data as an SCM with time-lags as follows
\begin{align}
  \begin{pmatrix}
    \COtwo_t\\
    \Temp_t
  \end{pmatrix}
  =
  \underbrace{\begin{pmatrix}
    0 & \beta\\
    \alpha & 0
  \end{pmatrix}}_{=B_0}
  \begin{pmatrix}
    \COtwo_t\\
    \Temp_t
  \end{pmatrix}
  + \sum_{k=1}^p B_k
  \begin{pmatrix}
    \COtwo_{t-k}\\
    \Temp_{t-k}
  \end{pmatrix}
  +
  \widetilde{S}_t,\label{eq:SVAR_climate}
\end{align}
where $\widetilde{S}_t=S_t+H_t$ with $S_t$ component-wise independent
non-stationary source signals and $H_t$ a stationary confounding
process. Vector-valued linear time-series models of this type are
referred to as structural auto regressive models (SVARs) \citep[see
e.g.,][]{lutkepohl2005}. They have been previously analyzed in the
confounding free-case by \citet{hyvarinen2010estimation}, using an ICA
based causal inference approach. A graphical representation of such a
model is shown in Supplement~\ref{sec:causal_appendix},
Figure~\ref{fig:graph}. In this example, we can think of the source
signals $S_t$ as being two independent summaries of important factors
that affect both temperature and carbon dioxide and vary over time,
e.g., environmental catastrophes like volcano eruptions and large
wildfires, sunspot activity or ice-coverage. These variations can be
considered as changing environmental conditions or interventions (see
Remark~\ref{rmk:intervention}). On the other hand the stationary
confounding process $H_t$ can be thought of as factors which affect
both temperature and carbon dioxide in a constant fashion over time,
for example this could be effects due the shifts in the earth's
rotation axis.

Assuming that this was the true underlying causal model, we could use
it to predict what happens under interventions.  From a climate
science perspective an interesting intervention is given by doubling
the concentration of $\text{CO}_2$ and determining the resulting
instantaneous (faster than 1000 years) effect on the temperature. This
effect is commonly referred to as equilibrium climate sensitivity
(ECS) due to $\text{CO}_2$ which is loosely defined as the change in
degrees temperature associated with a doubling of the concentration of
carbon dioxide in the earth's atmosphere. In the fifth assessment
report of the United Nations Intergovernmental Panel on Climate Change
it has been stated that "there is high confidence that ECS is
extremely unlikely less than 1~\degree C and medium confidence that
the ECS is likely between 1.5~\degree C and 4.5~\degree C and very
unlikely greater than 6~\degree C"
\citep[Chapter~10]{stocker2014climate}.  Since the measurement
frequency in our model is quite low (500 years) and we model the
logarithm of carbon dioxide the ECS corresponds to
\begin{equation*}
  \text{ECS} = \log(2)\alpha.
\end{equation*}
Estimating the model in \eqref{eq:SVAR_climate} can be done by first
fitting a vector auto-regressive model of the time lags using OLS
resulting in a vector of residuals
\begin{equation*}
  R_t=
  \begin{pmatrix}
    \COtwo_t\\
    \Temp_t
  \end{pmatrix} -
  \begin{pmatrix}
    \widehat{\COtwo_t}\\
    \widehat{\Temp_t}.
  \end{pmatrix}
\end{equation*}
Then, one can apply the two-step causal inference procedure described
in Section~\ref{sec:causalinterpretation} to
\begin{equation*}
R_t=B_0R_t+\widetilde{S}_t.
\end{equation*}
Since we are in a two-dimensional setting, step (ii) (i.e.,
identifying the causal parameters $\alpha$ and $\beta$ from the
estimated mixing matrix) only requires to assume that feedback loops
do not blow-up, which translates into $B_0$ having spectral norm less than one.
Given that the signal is sufficiently strong (i.e.,
there are sufficient interventions on both $\text{CO}_2$ and $T$), it
is possible to recover the causal parameters by trying both potential
permutations of the sources with subsequent scaling and assessing
whether the aforementioned condition is satisfied.

We applied this procedure based on \coroICAvar to the data in order to
estimate climate sensitivity and compared it with results obtained
when using fastICA or choiICA\,(var). The results are given in
Figure~\ref{fig:climate_results}.
\begin{figure}
  \centering
  \begin{minipage}{0.575\textwidth}
    \resizebox{\textwidth}{!}{
      \includegraphics{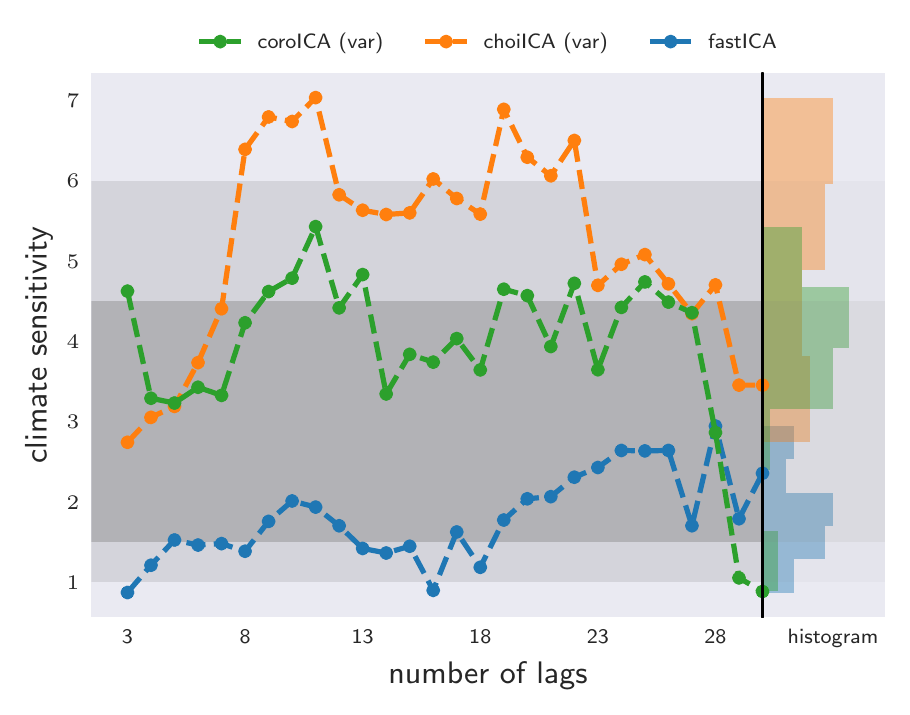}}
  \end{minipage}%
  \begin{minipage}{0.375\textwidth}
    \includegraphics{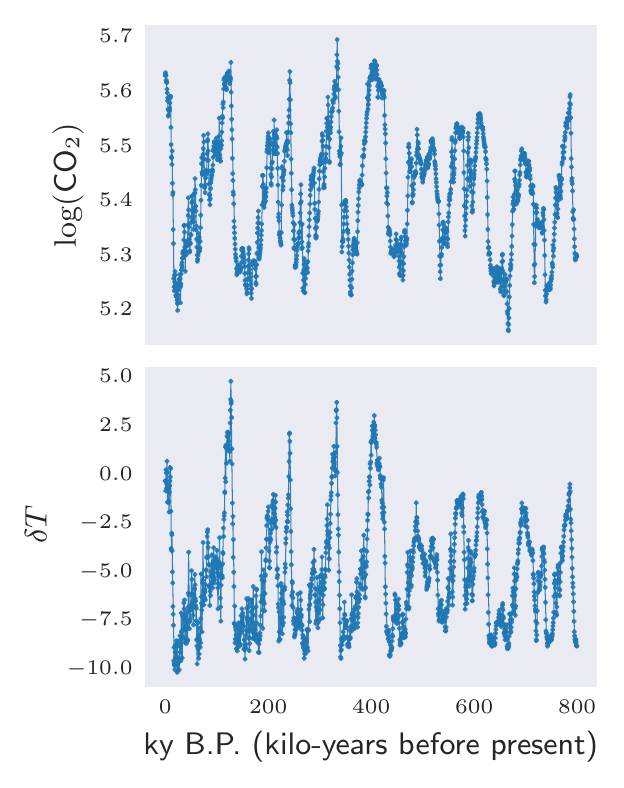}
  \end{minipage}
  \caption{(left) Estimated equilibrium climate sensitivity (ECS) for
    different ICAs depending on the number of lags included into the
    SVAR model.  The light gray and dark gray overlay indicate likely
    and very likely value ranges, respectively, for the true value of
    climate sensitivity as per the fifth assessment report of the
    United Nations Intergovernmental Panel on Climate Change (cf.\
    Section~\ref{ex:climate}).  The differences across procedures
    illustrate that the choice of ICA has a large effect on the
    estimation. (right) Interpolated time-series data, which we model
    with an SVAR model.}
  \label{fig:climate_results}
\end{figure}
We believe the results illustrate two important aspects. Firstly, the
choice of the lags has a strong effect on the estimation of the causal
effect parameters, particularly for boundary cases.  If it is chosen
too small the remaining time-dependence in the data can obscure the
signal. If it is chosen too big part of the signal starts being
removed. Choosing an appropriate number of lags is therefore
crucial. One option would be to apply an information criterion (AIC or
BIC) for this. Secondly, the results illustrate that the choice of ICA
has a large impact on the estimated causal effect parameters. More
specifically, both the assumed signal as well as the assumed
confounding have an impact on the estimation.  Compare the results
between fastICA (non-Gaussian signal) and choiICA/\coroICA (variance
signal) for the former and observe the differences between
fastICA/choiICA (no confounding) and \coroICA (adjusted for stationary
confounding) for the latter.  The choice of the ICA algorithm should
therefore be driven by the assumptions (both on signal type and
confounding) one is willing to employ on the underlying model.
Considering a variance signal and adjusting for confounding, \coroICA
appears to lead to estimates of equilibrium climate sensitivity that
are more closely in line with the highly likely bands previously
identified by the United Nations Intergovernmental Panel on Climate
Change.  This observation is only indicative as all three methods
yield highly variable results and also the panel's highly likely band
rests on certain assumptions that may become refuted at some later
point.  \coroICA can be considered a conservative choice if no
assumptions on confounding can be made, while noise-free methods may
outperform if indeed there were no confounding factors.

\section{Experiments}

In this section, we analyze empirical properties of \coroICA.  To this
end, we first illustrate the performance of \coroICA as compared to
time-concatenated versions of (noisy) ICA variants on simulated
data with and without confounding. We also compare on real data and
outline potential benefits of using our method when analyzing
multi-subject EEG data.

\subsection{Competing Methods}\label{sec:competing_methods}

In all of our numerical experiments, we apply \coroICA as outlined in
Algorithm~\ref{alg:coroICA}, where we partition each group based on
equally spaced grids and run a fixed number of $10\cdot 10^3$ iterations
of the uwedge approximate joint diagonalizer. Unless specified otherwise, \coroICA
refers to \coroICAvar (i.e., the variance signal based version)
and we explicitly write \coroICAvar, \coroICATD and \coroICAvarTD whenever appropriate to avoid confusion. We
compare with all of the methods in
Table~\ref{table:ica_methods}. Since no Python implementation was publicly available,
we implemented the choiICAs and SOBI methods ourselves also
based on a fixed number of $10\cdot 10^3$ iterations of the uwedge
approximate joint diagonalizer. For fastICA we use the implementation
from the scikit-learn Python library due to \citet{scikit-learn} and
use the default parameters.

For the simulation experiments in Section~\ref{sec:simulations},
we also compare to random projections of the
sources, where the unmixing matrix is simply sampled with iid standard
normal entries. The idea of this comparison is to give a baseline of
the unmixing problem and enhance intuition about the scores' behavior. In
order to illustrate the variance in this method, we generally sample
$100$ random projections and show the results for each of them.
A random mixing does not lead to interpretable sources, thus
we do not compare with random projections in the EEG
experiments in Section~\ref{sec:eeg_exps}.

\subsection{Simulations}\label{sec:simulations}

In this section, we investigate empirical properties of \coroICA in
well-controlled simulated scenarios.  First off, we show that we can
recover the correct mixing matrix given that the data is generated
according to our model \eqref{eq:mixing} and
Assumptions~\ref{assumption:group_structure} and~\ref{assumption:var}
hold, while the other ICAs necessarily fall short in this setting
(cf.\ Section~\ref{sec:sim_confounding}). Moreover, in
Section~\ref{sec:sim_robust} we show that even in the absence of any
confounding (i.e., when the data follows the ordinary ICA model and
$H\equiv 0$ in our model) we remain competitive with all competing
ICAs. Finally, in Section~\ref{sec:GARCH} we analyze the performance
of \coroICA for various types of signals and noise settings. Our first
two simulation experiments are based on block-wise shifting variance
signals, which we describe in \hyperlink{dat:datasimul}{Data~Set~1}
and our third simulation experiment is based on GARCH type models
described in \hyperlink{dat:datasimul2}{Data~Set~2}.

\subsubsection{Dependence on Confounding Strength}\label{sec:sim_confounding}

For this simulation experiment, we sample data according to
\hyperlink{dat:datasimul}{Data~Set~1} and choose to simulate $n=10^5$
(dimension $d=22$) samples from $m=10$ groups where each group
contains $n/m=10^4$ observations. Within each group, we select a
random partition consisting of $\abs{\mathcal{P}_g}=10$ subsets while
ensuring that these have the same size on average. We fix the signal
strength to $c_1=1$ and consider the behavior of \coroICA (trained on
half of the groups with an equally spaced grid of $10$ partitions per
group) for different confounding strengths
$c_1=\{0.125,0.25,0.5,1,1.5,2,2.5,3\}$. The results for $1000$
repetitions are shown in Figure~\ref{fig:sim_01}. To allow for a fair
comparison we take the same partition size for choiICA\,(var).

\begin{mdframed}[roundcorner=5pt, frametitle={\hypertarget{dat:datasimul}{Data Set 1}: Block-wise shifting variance signals}]
  For our simulations we select $m$ equally sized groups
  $\groups\coloneqq\{g_1,\dots, g_m\}$ of the data points
  $\{1,\dots, n\}$ and for each group $g\in\groups$ construct a
  partition $\mathcal{P}_g$. Then, we sample a model of the form
  \begin{equation*}
    X_i=A\cdot\left(S_{i}+C\cdot H_{i}\right),
  \end{equation*}
  where the values on the right-hand side are sampled as follows:
  \begin{itemize}
  \item $A, C\in\R^{d\times d}$ are sampled with
    iid entries from $\mathcal{N}(0, 1)$ and
    $\mathcal{N}(0, \frac{1}{d})$, respectively.
  \item For each $g\in\groups$ the variables $H_i\in\R^d$ are sampled
    from $\mathcal{N}(0, \sigma^2_g\identity_{d})$, where the
    $\sigma^2_g$ are sampled iid from $\operatorname{Unif}(0.1, b_1)$.
  \item For each $g\in\groups$ and $e\in\mathcal{P}_g$ the variables
    $S_i\in\R^d$ are sampled from $\mathcal{N}(0,
    \eta^2_e\identity_{d})$, where the $\eta^2_e$ are sampled iid from $\operatorname{Unif}(0.1, b_2)$.
  \end{itemize}
  The parameters $b_1$ and $b_2$ are selected in such a way that the
  expected confounding strength $c_1=\E(\sigma^2_g)$ and variance signal
  strength $c_2\coloneqq\E(\abs{\eta^2_e-\eta^2_f})$ are as dictated
  by the respective experiment. Due to the uniform distribution this
  reduces to
  \begin{equation*}
    b_1=2c_1-0.1\quad\text{and}\quad b_2=3c_2+0.1.
  \end{equation*}
\end{mdframed}
The results indicate that in terms of the MD index the competitors all
become worse as the confounding strength increases. All competing ICAs
systematically estimate an incorrect unmixing matrix. \coroICA on the
other hand only shows a very small loss in precision as confounding
increases; the small loss is expected due to the decreasing signal to
noise ratio. In terms of MCIS, the behavior is analogous but slightly
less well resolved; with increasing confounding strength the unmixing
estimation of all competing ICAs is systematically biased resulting in
bad separation of sources and high MCIS scores both out-of-sample and in-sample.

\begin{figure}[h]
  \includegraphics{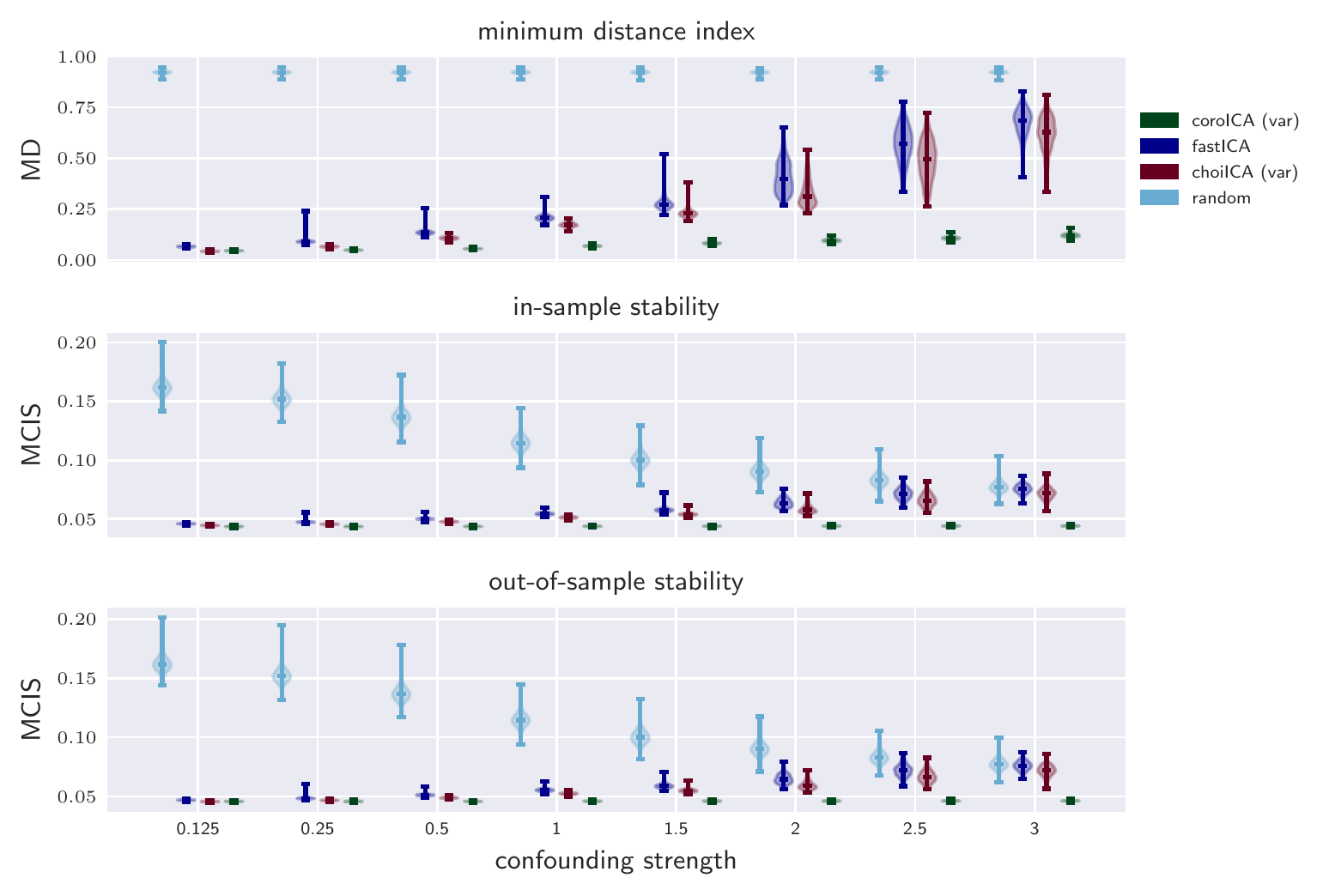}
  \caption{Results of the simulation experiment described in
    Section~\ref{sec:sim_confounding}. Plot shows performance measures
    (MD: small implies close to truth; MCIS: small implies stable) for
    fixed signal strength and various confounding strengths. The
    difference between the competing ICAs and \coroICA is more
    prominent for higher confounding strengths where the estimates of
    the competing ICAs are increasingly different from the true
    unmixing matrix and the sources become increasingly
    unstable.}\label{fig:sim_01}
\end{figure}

\subsubsection{Efficiency in Absence of Group Confounding}\label{sec:sim_robust}

For this simulation experiment, we sample data according to
\hyperlink{dat:datasimul}{Data~Set~1} and choose to simulate $n=2\cdot 10^4$
(dimension $d=22$) samples from $m=10$ groups where each group
contains $n/m=2\cdot 10^3$ observations. Within each group, we then
select a random partition consisting of
$\abs{\mathcal{P}_g}=10$ subsets while ensuring that these have the
same size on average. This time, to illustrate performance in the
absence of confounding, we fix the confounding strengths $c_1=0$ and
consider the behavior of \coroICA (applied to half of the groups with
an equally spaced grid of $10$ partitions per group) for different
signal strengths
$c_2=\{0.025, 0.05, 0.1, 0.2, 0.4, 0.8, 1.6, 3.2, 6.4\}$. The results
for $1000$ repetitions are shown in Figure~\ref{fig:sim_02}. Again,
choiICA\,(var) is applied with the same partition size.

The results indicate that overall \coroICA performs competitive in the
confounding-free case. In particular, there is no drastic negative hit
on the performance of \coroICA as compared to choiICA\,(var) in settings
where the data follows the ordinary ICA model. The slight advantage
compared to fastICA in this setting is due to the signal type which
favors ICA methods that focus on variance signals.

\begin{figure}
  \includegraphics{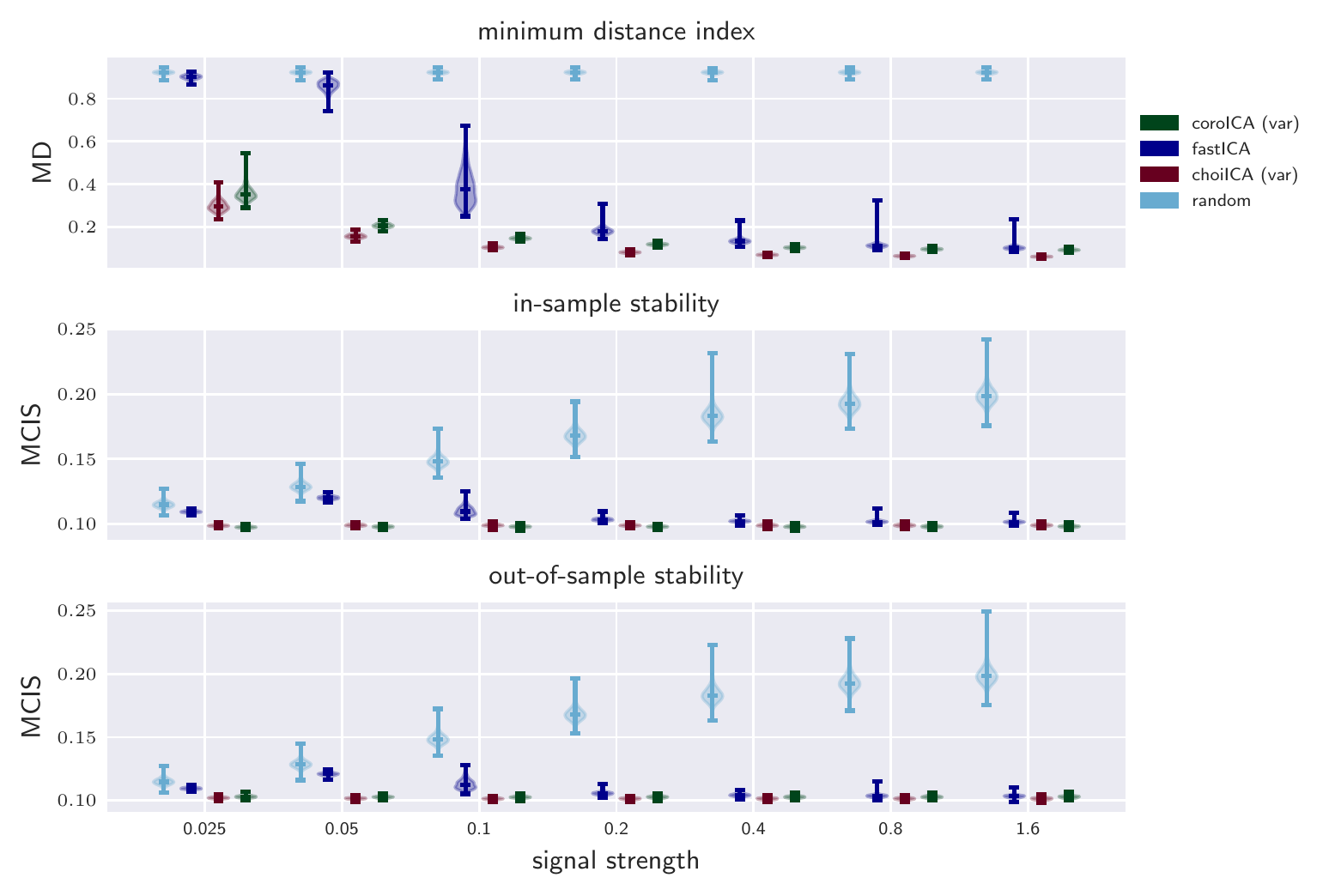}
  \caption{Results of the simulation experiment described in
    Section~\ref{sec:sim_robust}. Plot shows performance measures (MD:
    small implies close to truth; MCIS: small implies stability) for
    data generated without confounding and for various signal
    strengths. These results are reassuring, as they indicate that
    when applied to data that follows the ordinary ICA model, \coroICA
    still performs competitive to competing ICAs even though it allows
    for a richer model class.}\label{fig:sim_02}
\end{figure}

\subsubsection{Comparison with Other Noisy ICA
  Procedures}\label{sec:GARCH}

To get a better understanding of how our proposed ICA performs for
different signal and noise types, we compare it on simulated data as described
in \hyperlink{dat:datasimul2}{Data~Set~2}. We illustrate the different
behavior with respect to the different types of signal by applying all
three of our proposed \coroICA procedures (\coroICAvar,
\coroICATD and \coroICAvarTD) and compare them to the
corresponding choiICA variants which do not adjust for confounding
(choiICA\,(var), choiICA\,(TD) and choiICA\,(var\,\&\,TD)). While all
\coroICA procedures can deal with any type of stationary noise,
choiICA\,(TD) only works for time-independent noise and choiICA\,(var)
and choiICA\,(var\,\&\,TD) cannot handle any type of noise at all (see
Table~\ref{table:ica_methods}). Additionally, we also compare with
fastICA to assess its performance in the various noise settings. The
results are depicted in Figure~\ref{fig:GARCH}.

\begin{mdframed}[roundcorner=5pt,
  frametitle={\hypertarget{dat:datasimul2}{Data Set 2}: GARCH simulation}]
  For this simulation we consider different settings of the
  confounded mixing model
  \begin{equation*}
    X_t=A S_t + H_t.
  \end{equation*}
  More precisely, we consider the following three different
  GARCH type signals: (i)~changing variance, (ii)~changing
  time-dependence, and (iii)~both changing variance and changing
  time-dependence. For each of these signal types we consider two
  types of confounding (noise) terms: (a) time-independent and (b)
  time-dependent auto-regressive noise. For both we construct $d$ independent processes
  $\tilde{H}^1,\dots,\tilde{H}^d$ and then combine them with a
  random mixing matrix $C$ as follows
  \begin{equation*}
    H_t=C\cdot\tilde{H}_t.
  \end{equation*}
  Full details are given in Supplement~\ref{sec:appendix_simu}.
\end{mdframed}
\begin{figure}[h]
  \centering
  \includegraphics{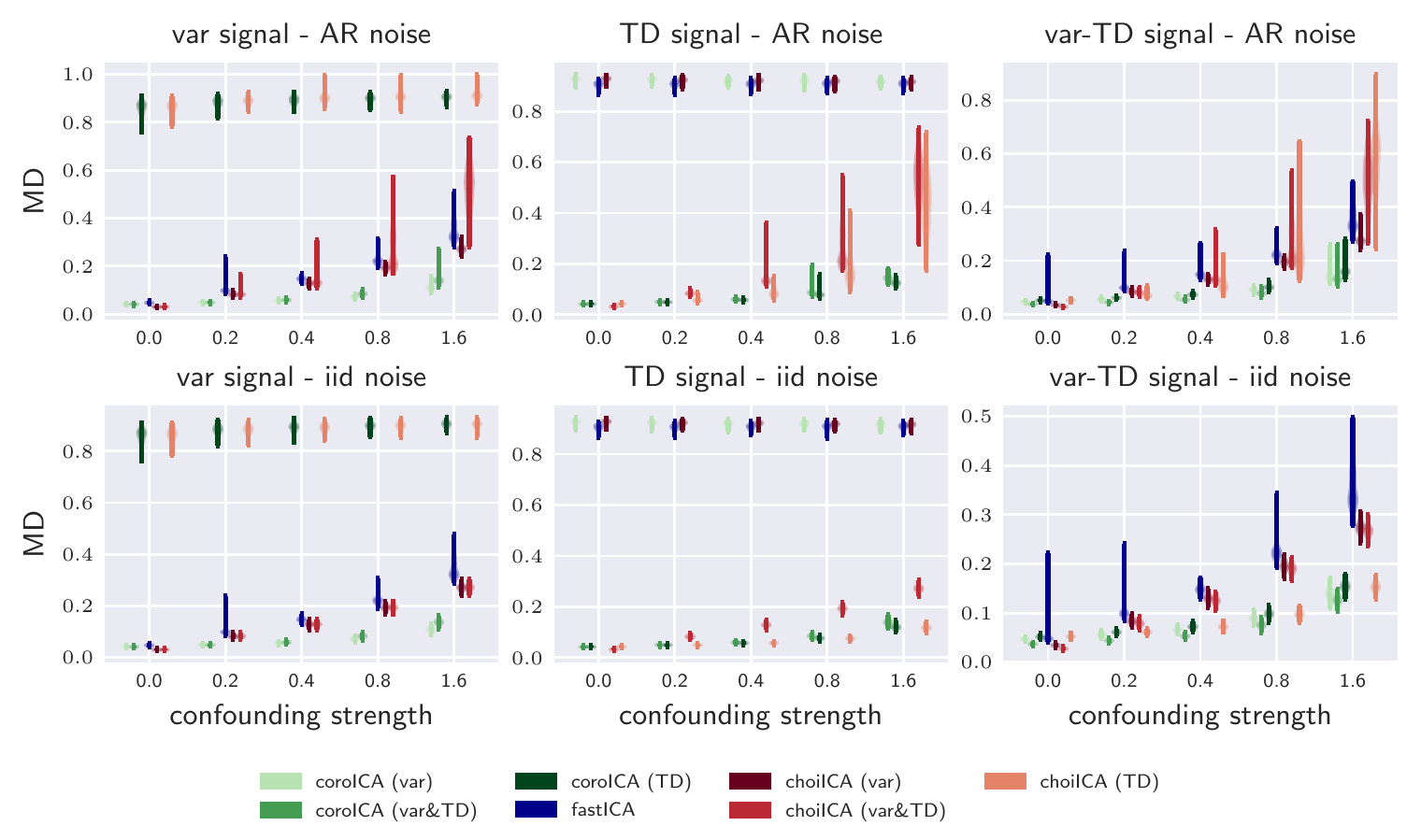}
  \caption{Results of the simulation experiment described in
    Section~\ref{sec:GARCH} and
    \protect\hyperlink{dat:datasimul2}{Data~Set~2}.  Plots show performance
    (MD: small implies close to truth) for data generated with
    auto-regressive (AR) or iid noise and for var, TD, and var \& TD
    signal as described in \protect\hyperlink{dat:datasimul2}{Data~Set~2}.
    \coroICAvarTD is able to estimate the correct mixing in all
    of the considered settings, while others break whenever the more
    restrictive signal/noise assumptions are not met.}
  \label{fig:GARCH}
\end{figure}
In all settings the most general method \coroICAvarTD
is able to estimate the correct mixing.
The two signal specific methods \coroICATD and
\coroICAvar are also able to accurately estimate the mixing in
settings where a corresponding signal exists. It is also worth noting
that they slightly outperform \coroICAvarTD in these
settings. In contrast, when comparing with the choiICA variants, \coroICA is in
general able to outperform the corresponding method. Only in the
setting of a changing time-dependence with time-independent noise,
choiICA\,(TD) is able to slightly outperform \coroICATD.

\subsubsection{Summary of the Performance of \coroICA}

In summary, \coroICA performs well on a larger model class consisting
of both the group-wise confounded as well as the confounding-free
case. An advantage over all competing ICAs is gained
in confounded settings (as shown in Section~\ref{sec:sim_confounding})
while there is at most a small disadvantage in the unconfounded case
(cf. Section~\ref{sec:sim_robust}).  This suggests that whenever the
data is expected to contain at least small amounts of stationary noise
or confounding, one may be better off using \coroICA as the richer
model class will guard against wrong results.  The results in
Section~\ref{sec:GARCH} further underline the robustness of our
proposed method to various types of noise (and signals) for which
other methods break.  Again, even in settings that satisfy the assumptions of
the more tailored methods \coroICA remains competitive.

\subsection{EEG Experiments}\label{sec:eeg_exps}

ICA is often applied in the analysis of EEG data.
Here, we illustrate the potential benefit and use of \coroICA for this.
Specifically, we consider a multi-subject EEG experiment as
depicted in Figure~\ref{fig:eegscenario}. The goal is to find a single
mixing matrix that separates the sources simultaneously on all
subjects. Our proposed model allows that the EEG recordings
for each subject have a different but stationary noise term $H$.
\begin{figure}[h]
\begin{tikzpicture}[scale=0.95]
    \node (A) at(4.1,6) {\textbf{subject a}};
    \node (Aform) at(4.1,5.5) {$X_a = AS_a + H_a$};
    \node (Abrain) at(4.2,-.2){{\reflectbox{\includegraphics[keepaspectratio,width=4.5cm]{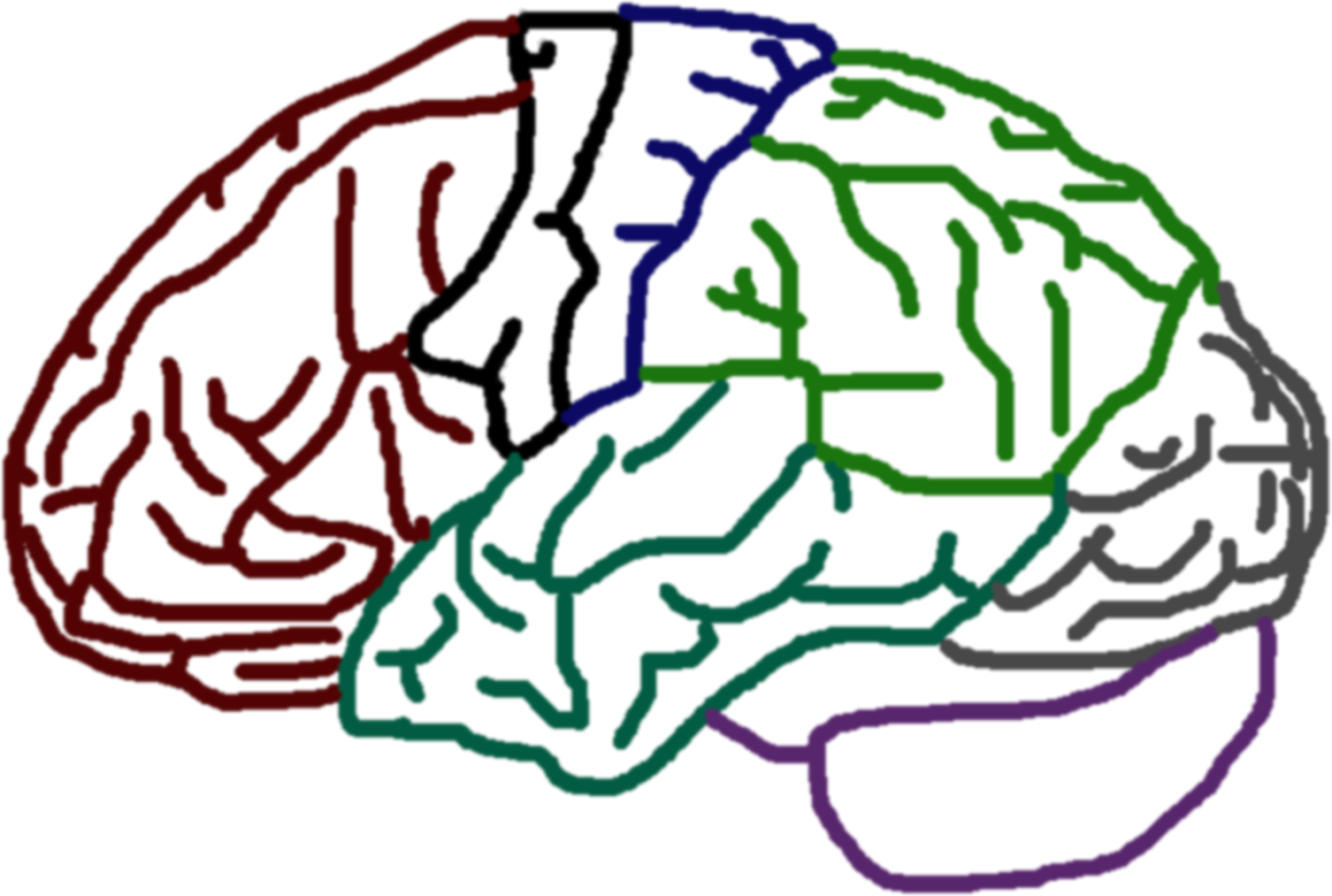}}}};

    \node[nodec] (As1) at(2.6,-.7) {$S^1_a$};
    \node[nodec] (As2) at(4.1,.2) {$S^2_a$};
    \node[nodec] (As3) at(5.9,-.4) {$S^3_a$};

    \node[node] (Ax1) at(2.6,2.5) {$X^1_a$};
    \node[node] (Ax2) at(4.25,2.5) {$X^2_a$};
    \node[node] (Ax3) at(5.9,2.5) {$X^3_a$};

    \node[nodeh] (Ah1) at(2.6,4.5) {$H^1_a$};
    \node[nodeh] (Ah2) at(4.25,4.5) {$H^2_a$};
    \node[nodeh] (Ah3) at(5.9,4.5) {$H^3_a$};

    \draw[->] (Ah1) -- (Ax1);
    \draw[->] (Ah1) -- (Ax2);
    \draw[->] (Ah2) -- (Ax2);
    \draw[->] (Ah2) -- (Ax3);

    \draw[waves,segment angle=6]   (As1) -- (Ax1);
    \draw[waves,segment angle=4.5] (As1) -- (Ax2);
    \draw[waves,segment angle=3]   (As1) -- (Ax3);
    \draw[waves,segment angle=4.5] (As2) -- (Ax1);
    \draw[waves,segment angle=6]   (As2) -- (Ax2);
    \draw[waves,segment angle=4.5] (As2) -- (Ax3);
    \draw[waves,segment angle=3]   (As3) -- (Ax1);
    \draw[waves,segment angle=4.5] (As3) -- (Ax2);
    \draw[waves,segment angle=6]   (As3) -- (Ax3);

    \node (B) at(10,6) {\textbf{subject b}};
    \node (Bform) at(10,5.5) {$X_b = AS_b + H_b$};
    \node (Bbrain) at(10.1,-.2){{\reflectbox{\includegraphics[keepaspectratio,width=4.5cm]{brain}}}};

    \node[nodec] (Bs1) at(8.5,-.7) {$S^1_b$};
    \node[nodec] (Bs2) at(10,.2) {$S^2_b$};
    \node[nodec] (Bs3) at(11.8,-.4) {$S^3_b$};

    \node[node] (Bx1) at(8.5,2.5) {$X^1_b$};
    \node[node] (Bx2) at(10.15,2.5) {$X^2_b$};
    \node[node] (Bx3) at(11.8,2.5) {$X^3_b$};

    \node[nodeh] (Bh1) at(8.5,4.5) {$H^1_b$};
    \node[nodeh] (Bh2) at(10.15,4.5) {$H^2_b$};
    \node[nodeh] (Bh3) at(11.8,4.5) {$H^3_b$};

    \draw[->] (Bh1) -- (Bx1);
    \draw[->] (Bh2) -- (Bx1);
    \draw[->] (Bh2) -- (Bx2);
    \draw[->] (Bh2) -- (Bx3);
    \draw[->] (Bh3) -- (Bx2);
    \draw[->] (Bh3) -- (Bx3);

    \draw[waves,segment angle=6]   (Bs1) -- (Bx1);
    \draw[waves,segment angle=4.5] (Bs1) -- (Bx2);
    \draw[waves,segment angle=3]   (Bs1) -- (Bx3);
    \draw[waves,segment angle=4.5] (Bs2) -- (Bx1);
    \draw[waves,segment angle=6]   (Bs2) -- (Bx2);
    \draw[waves,segment angle=4.5] (Bs2) -- (Bx3);
    \draw[waves,segment angle=3]   (Bs3) -- (Bx1);
    \draw[waves,segment angle=4.5] (Bs3) -- (Bx2);
    \draw[waves,segment angle=6]   (Bs3) -- (Bx3);

    \node (Morebrain) at(15.1,-.2) {{\Huge $\cdots$}};

    \node (coroICA) at(15.1,1.3) {{$\coroICA(X_a, X_b, ...) \approx A$}};
\end{tikzpicture}
\caption{Illustration of a multi-subject EEG recording. For each subject, EEG
  signals $X$ are recorded which are assumed to be corrupted by
  subject-specific (but stationary) noise terms $H$. The goal is to
  recover a single mixing matrix $A$ that separates signals well across all
  subjects.}\label{fig:eegscenario}
\end{figure}
We illustrate the applicability of our method to this setting based on
two publicly available EEG data sets.
\begin{mdframed}[roundcorner=5pt,
  frametitle={\hypertarget{dat:covertattention}{Data Set 3}: CovertAttention data}]
  This data set is due to \citet{treder2011} and consists of EEG
  recordings of 8 subjects performing multiple trials of covertly
  shifting visual attention to one out of 6 cued directions.  The data
  set contains recordings of
  \begin{itemize}
  \item 8 subjects,
  \item for each subject there exist 6 runs with 100 trials,
  \item each recording consists of 60 EEG channels recorded at 1000 Hz
    sampling frequency, while we work with the publicly available data
    that is downsampled to 200 Hz.
  \end{itemize}
  Since visual inspection of the data revealed data segments with huge
  artifacts and details about how the publicly available data was
  preprocessed was unavailable to us, we removed outliers and
  high-pass filtered the data at 0.5 Hz.  In particular, along each
  dimension we set those values to the median along its dimension that
  deviate more than 10 times the median absolute distance from this
  median. We further preprocess the data by re-referencing to common
  average reference (car) and projecting onto the orthogonal
  complement of the null component. For our unmixing estimations, we
  use the entire data, i.e., including intertrial breaks.

  For classification experiments (cf.\
  Section~\ref{sec:EEG.classification}) we use, in line with
  \citet{treder2011}, the 8--12 Hz bandpass-filtered data during the
  500--2000 ms window of each trial, and use the log-variance as
  bandpower feature~\citep{lotte2018review}.
  The classification analysis is restricted to valid trials
  (approximately 311 per subject) with the
  desired target latency as described in \citet{treder2011}.
\end{mdframed}
Results on the CovertAttention
\hyperlink{dat:covertattention}{Data~Set~3} are presented here, while
the results of the analogous experiments on the BCICompIV2a
\hyperlink{dat:bcicomp}{Data~Set~4} are deferred to
Supplement~\ref{sec:EEG_results}.  For both data sets, we compare the
recovered sources of \coroICA with those recovered by competing ICA
methods.  Since ground truth is unknown we report comparisons based on
the following three criteria:

\begin{description}
\item[stability and independence]\ \\
We use MCIS (cf.\ Section~\ref{sec:scores}) to assess the stability and
independence of the recovered sources both in- and out-of-sample.

\item[classification accuracy]\ \\
  For both data sets there is label
  information available that associates certain time windows of the
  EEG recordings with the task the subjects were performing at that
  time. Based on the recovered sources, we build a
  classification pipeline relying on feature extraction and
  classification techniques that are common in the
  field~\citep{lotte2018review}. The achieved classification accuracy
  serves as a proxy of how informative and suitable the extracted signals are.

\item[topographies]\ \\
  For a qualitative assessment, we inspect the topographic maps of the
  extracted sources, as well as the corresponding power spectra and a
  raw time-series chunk. This is used to illustrate that the sources
  recovered by \coroICA do not appear random or implausible for EEG
  recordings and are qualitatively similar to what is expected from
  other ICAs.  Furthermore, we provide an overview over all components
  achieved on \hyperlink{dat:covertattention}{Data~Set~3} by SOBI,
  fastICA, and \coroICA in the Supplementary
  Section~\ref{sec:alltopos}, where components are well resolved when
  the corresponding topographic map and activation map are close to
  each other (cf.\ Section~\ref{sec:scores}).
\end{description}

\subsubsection{Stability and Independence}\label{sec:stability}

We aim to probe stability not only in-sample but also verify the
expected increase in stability when applying the unmixing matrix to
data of new unseen subjects, i.e., to new groups of samples with
different confounding specific to that subject. In order to assess
stability and independence of the recovered sources in terms of the
MCIS both in- and out-of-sample and for different amounts of training
samples, we proceed by repeatedly splitting the data into a training
and a test data set.  More precisely, we construct all possible splits
into training and test subjects for any given number of training
subjects.  For each pair of training and test set, we fit an unmixing
matrix using \coroICA and all competing methods described in
Section~\ref{sec:competing_methods}. We then compute the MCIS on the
training and test data for each method separately and collect the
results of each training-test split for each number of training
subjects.

Results obtained on the CovertAttention data set (with equally spaced
partitions of $\approx$15 seconds length) are given in
Figure~\ref{fig:instability_covert} and the results for the
BCICompIV2a data set (with equally spaced partitions of $\approx$15
seconds length) are shown in Supplement~\ref{sec:appendix_stability},
Figure~\ref{fig:instability_bci}.  For both data sets the results are
qualitatively similar and support the claim that the unmixing obtained
by \coroICA is more stable when transferred to new unseen subjects.
While for the competing ICAs the instability on held-out subjects does
not follow a clear decreasing trend with increasing number of training
subjects, \coroICA can successfully make use of additional training
subjects to learn a more stable unmixing matrix.

\begin{figure}[h]
\centering
\includegraphics{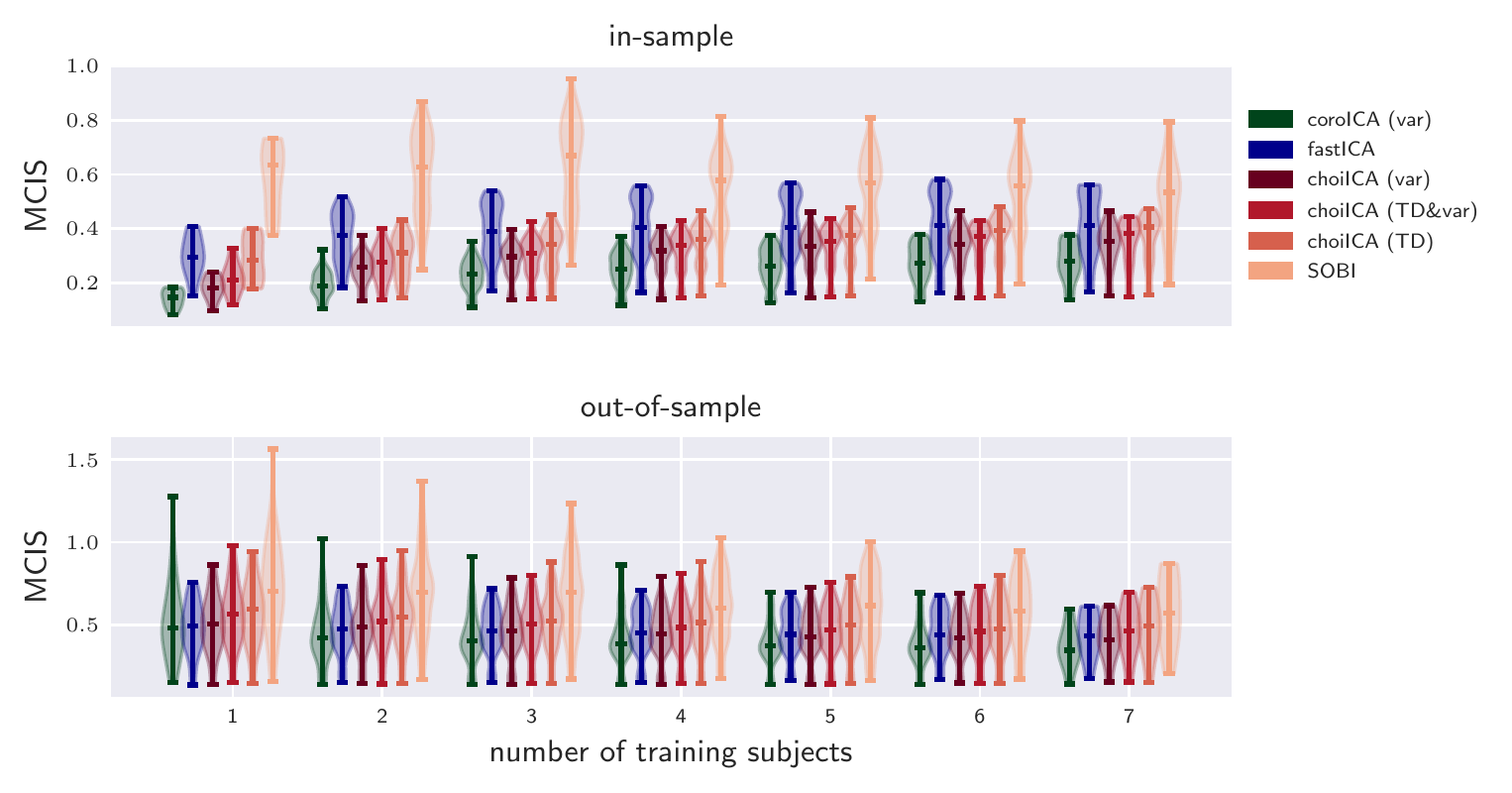}
\caption{Experimental results for comparing the stability of sources
  (MCIS: small implies stable) trained on different numbers of
  training subjects (cf.\ Section~\ref{sec:stability}), here on the
  CovertAttention \protect\hyperlink{dat:covertattention}{Data~Set~3},
  demonstrating that \coroICA, in contrast to the competing ICA
  methods, can successfully incorporate more training subjects to
  learn more stable unmixing matrices when applied to new unseen
  subjects.}
  \label{fig:instability_covert}
\end{figure}

\begin{figure}[h]
\centering
  \includegraphics{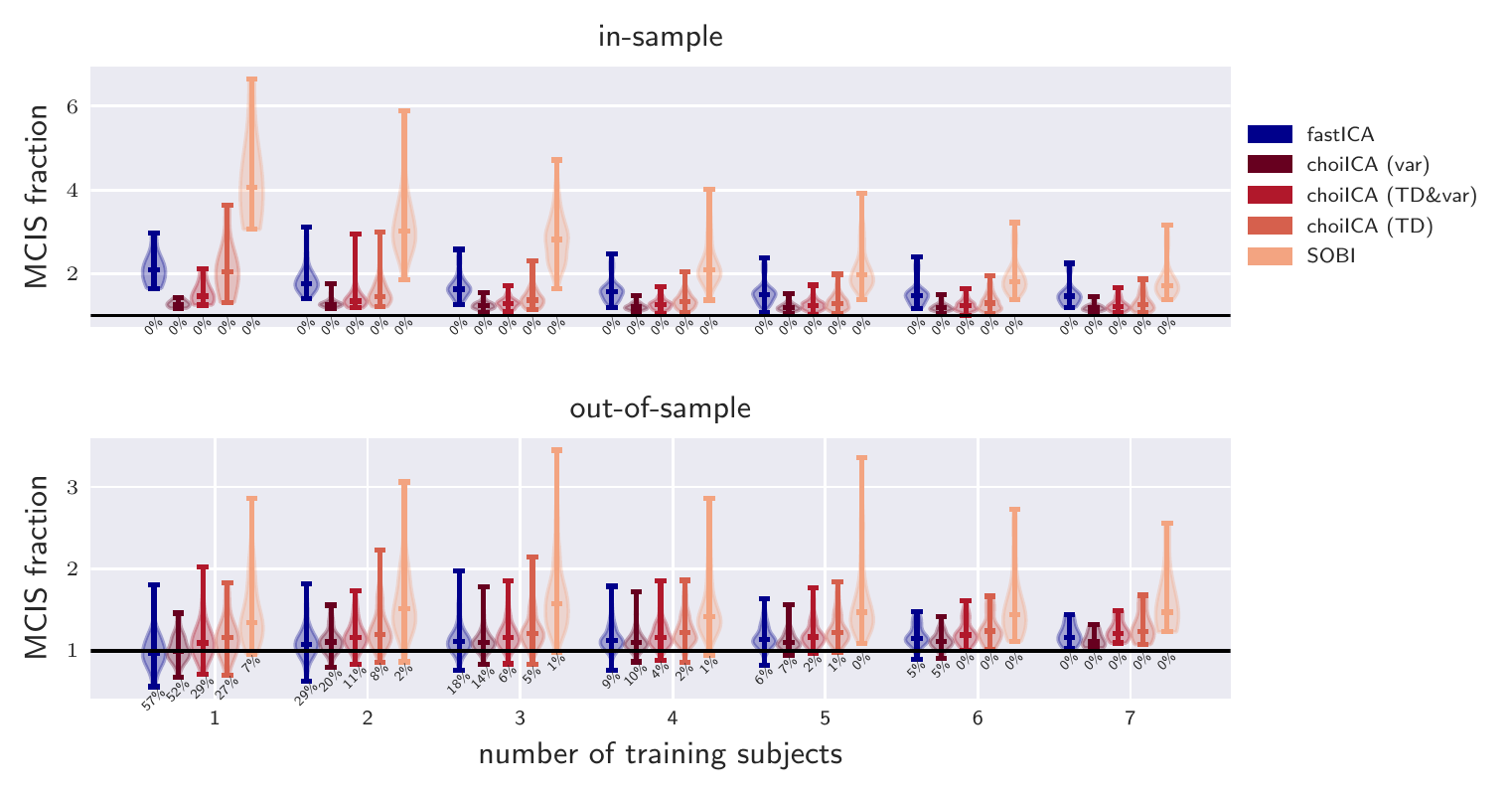}
  \caption{Experimental results for comparing the stability of sources
    of the competing methods relative to the stability obtained by
    \coroICA (MCIS fraction: above $1$ implies less stable than
    \coroICA) trained on different numbers of training subjects (cf.\
    Section~\ref{sec:stability}), here on the CovertAttention
    \protect\hyperlink{dat:covertattention}{Data~Set~3}, demonstrating
    that \coroICA can successfully incorporate more training subjects
    to learn more stable unmixing matrices when applied to new unseen
    subjects.}
  \label{fig:instability_fraction_covert}
\end{figure}

Due to the characteristics and low signal-to-noise ratio in EEG recordings,
the evaluation based on the absolute MCIS score is less well resolved
than what we have seen in the simulations before.  For this reason we
additionally provide a more focused evaluation by considering the MCIS
fraction: the fraction of the MCIS achieved on a subject by the respective
competitor method divided by the MCIS achieved on that subject by
\coroICA when trained on the same subjects. Thus, this score compares
MCIS on a per subject basis, where values greater than $1$ indicate
that the respective competing ICA method performed worse than
\coroICA.  Figure~\ref{fig:instability_fraction_covert} shows the
results on the CovertAttention \hyperlink{dat:covertattention}{Data~Set~3}
confirming that \coroICA can successfully incorporate more training
subjects to derive a better unmixing of signals.

\subsubsection{Classification based on Recovered Sources}\label{sec:EEG.classification}

While the results in the previous section indicate that \coroICA can
lead to more stable separations of sources in EEG than the competing
methods, in scenarios with an unknown ground truth the stability of
the recovered sources cannot serve as the sole determining criterion
for assessing the quality of recovered sources.  In addition to asking
whether the recovered sources are stable and independent variance
signals, we hence also need to investigate whether the sources
extracted by \coroICA are in fact reasonable or meaningful. In the
``America's Got Talent Duet Problem'' (cf.\
Example~\ref{ex:toy_example}) this means that each of the recovered
sources should only contain the voice of one (independent) singer
(plus some confounding noise that is not the other singer).  For EEG
data, this assessment is not as easy.  Here, we approach this problem
from two angles: (a) in this section we show that the recovered
sources are informative and suitable for common EEG classification
pipelines, (b) in Section~\ref{sec:topographic_maps} we qualitatively
assess the extracted sources based on their power spectra and
topographic maps.

In both data sets there are labeled trials, i.e., segments of data
during which the subject covertly shifts attention to one of six cues
(cf.\ \hyperlink{dat:covertattention}{Data~Set~3}) or performs one of
four motor imagery tasks (cf.\
\hyperlink{dat:bcicomp}{Data~Set~4}). Based on these, one can try to
predict the trial label given the trial EEG data. To mimic a situation
where the sources are transferred from other subjects, we assess the
informativeness of the extracted sources in a leave-k-subjects-out
fashion as follows. We estimate an unmixing matrix on data from all
but $k$ subjects, compute bandpower features for each extracted signal
and for each trial (as described in
\hyperlink{dat:covertattention}{Data~Set~3}
and~\hyperlink{dat:bcicomp}{Data~Set~4}), and on top of those we train
an ensemble of $200$ bootstrapped shrinkage linear discriminant
analysis classifiers where each is boosted by a random forest
classifier on the wrongly classified trials.  This pipeline (signal
unmixing, bandpower-feature computation, trained ensemble classifier),
is then used to predict the trials on the $k$ held-out subjects.

The results are reported in Figure~\ref{fig:classification_covert} and
Supplement~\ref{sec:appendix_EEG.classification},
Figure~\ref{fig:classification_bci} which show for each number of
training subjects, the accuracies achieved on the respective held-out
subjects when using the unmixing obtained on the remaining subjects by
either \coroICA or one of the competitor methods.  The results on both
data sets support the claim that the sources recovered by \coroICA are
not only stable but in addition also capture meaningful aspects of the
data that enable competitive classification accuracies in
fully-out-of-sample classification. The mean improvement in
classification accuracy of \coroICA over the other methods increases
with increasing number of training subjects. This behavior is expected
since it is difficult to disambiguate signal from subject-specific
confounding for few training subjects, while \coroICA is expected to
learn an unmixing which better adjusts for the confounding with more
training subjects.

\begin{figure}[h]
  \includegraphics{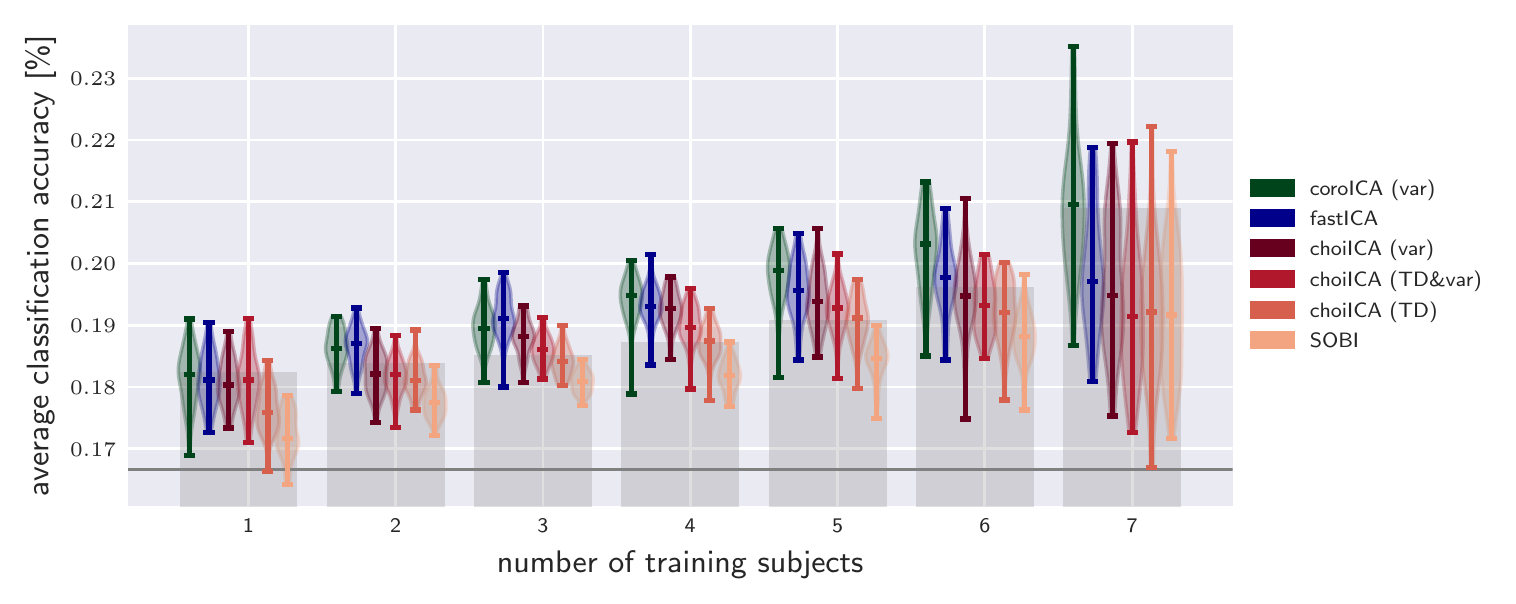}
  \caption{Classification accuracies on held-out subjects (cf.\
    Section~\ref{sec:EEG.classification}), here on the CovertAttention
    \protect\hyperlink{dat:covertattention}{Data~Set~3}. Gray regions indicate
    a 95\% confidence interval of random guessing accuracies.}
  \label{fig:classification_covert}
\end{figure}

It is worth noting that these classification results depend heavily on
the employed classification pipeline subsequent to the source
separation. Here, our goal is only to show that \coroICA does indeed
separate the data into informative sources.  In practice, and when
only classification accuracy matters, one might also consider using a
label-informed source separation \citep{dahne2014spoc}, employ common
spatial patterns \citep{koles1990spatial} or use decoding techniques
based on Riemannian geometry \citep{barachant2012multiclass}.

\subsubsection{Topographic Maps}\label{sec:topographic_maps}

The components that \coroICA extracts from EEG signals are stable
(cf.\ Section~\ref{sec:stability}) and meaningful in the sense that
they contain information that enables classification of trial labels,
which is a common task in EEG studies (cf.\
Section~\ref{sec:EEG.classification}). In this section, we complement
the assessment of the recovered sources by demonstrating that the
results obtained by \coroICA lead to topographies, activation maps,
power spectra and raw time-series that are similar to what is commonly
obtained during routine ICA analyses of EEG data when the plausibility
and nature of ICA components is to be judged.

Topographies are common in the EEG literature to depict the relative
projection strength of extracted sources to the scalp sensors. More
precisely, the column-vector $a_j$ of $A = V^{-1}$ that specifies the
mixing of the $j$-th source component
is visualized as follows.  A sketched top view of the head is
overlayed with a heatmap where the value at each electrodes' position
is given by the corresponding entry in $a_j$. These topographies are
indicative of the nature of the extracted sources, for example the
dipolarity of source topographies is a criterion invoked to identify
cortical sources~\citep{delorme2012} or the topographies reveal that
the source mainly picks up changes in the electromagnetic field
induced by eye movements.  Another way to visualize an extracted
source is an activation map, which is commonly obtained by depicting the vector
$\empcov(X)v_j^\top$ (where $v_j$ is $j$-th row of unmixing matrix
$V$) and shows for each electrode how the signal observed at that
electrode covaries with the signal extracted by
$v_j$~\citep{haufe2014interpretation}.  Besides inspecting the raw
time-series data, another criterion invoked to separate cortical from
muscular components is the log power spectrum.  For example, a
monotonic increase in spectral power starting at around 20 Hz is
understood to indicate muscular
activity~\citep{goncharova2003emg} and peaks in typical EEG frequency
ranges are used to identify brain-related components.\footnote{These
  are commonly employed criteria which are also advised in
  the eeglab tutorial
  \citep[\url{https://sccn.ucsd.edu/wiki/Chapter_09:_Decomposing_Data_Using_ICA}]{delorme2004eeglab}
  and the neurophysiological biomarker toolbox wiki
  \citep[\url{https://www.nbtwiki.net/doku.php?id=tutorial:how_to_use_ica_to_remove_artifacts}]{hardstone2012detrended}.}.

In Figure~\ref{fig:topomap}, we depict the aforementioned criteria for
three exemplary components extracted by \coroICA on the
CovertAttention \hyperlink{dat:covertattention}{Data~Set~3}.
\begin{figure}
  \includegraphics{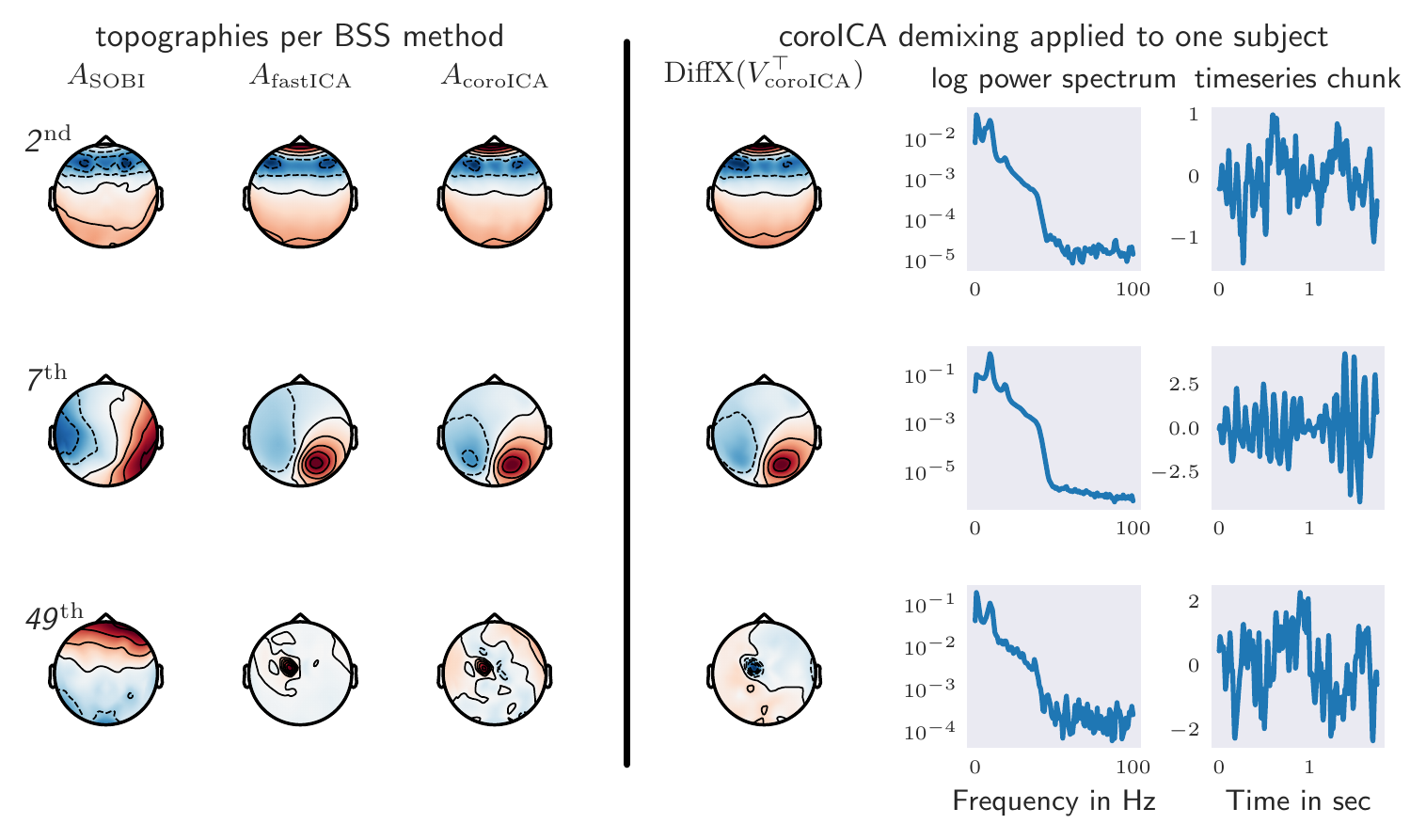}
  \caption{Visualization of exemplary EEG components recovered on the
    CovertAttention
    \protect\hyperlink{dat:covertattention}{Data~Set~3}. On the left
    the topographies of three components are shown where the mixing
    matrix is the inverse of the unmixing matrix obtained by SOBI
    ($A_\text{SOBI}$), the unmixing matrix obtained by fastICA
    ($A_\text{fastICA}$) and that of \coroICAvar ($A_\text{coroICA}$). On
    the right we depict, for a randomly chosen subject, the activation
    maps (cf.\ Section~\ref{sec:topographic_maps} and
    \ref{sec:scores}), the log power spectra, and randomly chosen
    chunks of the raw time-series data corresponding to the respective
    \coroICAvar components. Components extracted by \coroICAvar are
    qualitatively similar to those of the commonly employed ICA
    procedures; see Section~\ref{sec:topographic_maps} for details.}
  \label{fig:topomap}
\end{figure}
Following the discussion in Section~\ref{sec:scores} we show the
activation maps as
\[\operatorname{DiffX}(v_j^\top) = \sum_{\mathcal{M}\in\mathcal{M}^*} \operatorname{sign}(v_jMv_j^\top) Mv_j^\top,\]
which captures variance changing signal and allows to asses the quality of
a recovered source by comparison to the topographic map $a_j$ (cf.\ Equation~\ref{eq:activation_map}).
Here, the idea is to demonstrate that
\coroICA components are qualitatively similar to
components extracted by commonly employed SOBI-ICA or fastICA.
Therefore, we
choose to display one example of an ocular component
(2\textsuperscript{nd} where the topography is indicative of eye
movement), a cortical component (7\textsuperscript{th} where the
dipolar topography, the typical frequency peak at around 8--12 Hz, and
the amplitude modulation visible in the raw time-series are indicative
of the cortical nature), and an artifactual component
(51\textsuperscript{st} where the irregular topography and the high
frequency components indicate an artifact). For comparison, we
additionally show for each component the topographies of the
components extracted by SOBI-ICA or fastICA by matching
the recovered source which most strongly correlates with the one
extracted by \coroICA.
The components extracted by \coroICA closely resemble the results
one would obtain from a commonly employed ICA analysis on EEG data.

For completeness, we provide an overview over all components extracted
on \hyperlink{dat:covertattention}{Data~Set~3} by SOBI, fastICA, and
\coroICAvar in the Supplementary Section~\ref{sec:alltopos}.  Components
are well resolved when the corresponding topographic map and
activation map are close to each other (cf.\
Section~\ref{sec:scores}), which, by visual inspection, appears to be
more often the case for \coroICA than for the competing methods.

\section{Conclusion}

In this paper, we propose a method for recovering independent sources
corrupted by group-wise stationary confounding. It extends ordinary
ICA to an easily interpretable model, which we believe is relevant for
many practical problems as is demonstrated in Section~\ref{ex:climate}
for climate data and Section~\ref{sec:eeg_exps} for EEG data. We give
explicit assumptions under which the sources are identifiable in the
population case (cf.\ Section~\ref{sec:identifiable}). Moreover, we
introduce a straightforward algorithm for estimating the sources based
on the well-understood concept of approximate joint matrix
diagonalization. As illustrated in the simulations in
Section~\ref{sec:simulations}, this estimation procedure performs
competitive even for data from an ordinary ICA model, while
additionally being robust and able to adjust for group-wise stationary
confounding. For real data, we show that the \coroICA model indeed
performs reasonably on EEG data and leads to improvements in
comparison to commonly employed approaches, while at the same time
preserving an enhanced interpretation of the recovered sources.

\acks
The authors thank Vinay Jayaram, Nicolai Meinshausen, Jonas Peters,
Gian Thanei, the action editor Kenji Fukumizu, and anonymous reviewers
for helpful discussions and constructive comments. NP and PB were
partially supported by the European Research Commission grant
786461 CausalStats - ERC-2017-ADG.

\newpage
\renewcommand\appendixpagename{Supplementary material}
\renewcommand\appendixtocname{Supplementary material}
\appendix
\appendixpage
The supplementary material consists of the following appendices.
\begin{enumerate}
\item[\textbf{A}]
  \hyperref[sec:proofs]{\textbf{Identifiability Proof}}
\item[\textbf{B}]
  \hyperref[sec:complementary]{\textbf{Complementary Material}}
\item[\textbf{C}]
  \hyperref[sec:EEG_results]{\textbf{EEG Experiments on \hyperlink{dat:bcicomp}{Data~Set~4}}}
\item[\textbf{D}]
  \hyperref[sec:alltopos]{\textbf{All Topographies on \hyperlink{dat:covertattention}{Data~Set~3}}}
\end{enumerate}

\section{Identifiability Proof}\label{sec:proofs}

The proof is based on Theorem~1 from \citet{kleinsteuber2013}. For completeness,
we introduce some of the notation therein and
state their result with adapted notation to ease following our proof of Theorem~\ref{thm:mixing}.
We begin by defining the
empirical correlation between two vectors $\mathbf{v},\mathbf{w}\in\R^{d}$ as
\begin{equation*}
  \empcorr(\mathbf{v}, \mathbf{w})\coloneqq
  \begin{cases}
    \frac{\mathbf{v}^{\top}\mathbf{w}}{\norm{\mathbf{v}}\norm{\mathbf{w}}},
    \quad&\text{if }\mathbf{v}\not = 0 \text{ and } \mathbf{w}\not =
    0,\\
    1,\quad&\text{otherwise.}
  \end{cases}
\end{equation*}
Moreover, for a collection of $(d\times d)$-real diagonal matrices
$\{Z_1,\dots,Z_m\}$, we define the following collinearity measure
\begin{equation}
  \label{eq:rho_def}
  \rho(Z_1,\dots,Z_m)\coloneqq\max_{1\leq k < l\leq d}\abs{\empcorr(\mathbf{z}_k,\mathbf{z}_l)},
\end{equation}
where $\mathbf{z}_{j}\coloneqq(z_{1}(j)\dots,z_{m}(j))$ and
$z_{i}(j)$ is the $j$-th diagonal element of the matrix $Z_i$. Using
this notation we can state the uniqueness result due to \citet[Theorem
1]{kleinsteuber2013} as follows.

\begin{theorem}[{{\citet[Theorem 1]{kleinsteuber2013}}}]
  \label{thm:kleinsteuber}
  Let $D_i\in\R^{d\times d}$, for $i\in\{1,\dots,m\}$ be diagonal, and
  let $M\in\R^{d\times d}$ be an invertible matrix so that
  $M^{\top}D_iM$ is diagonal as well. Then $M$ is essentially, up to scaling and permutation of its columns, unique
  if and only if $\rho(D_1,\dots,D_m)<1$.
\end{theorem}

\-\\ Using this result we prove Theorem~\ref{thm:mixing}.\-\\

\begin{proof}
  The theorem is proven by the correct invocation of
  Theorem~\ref{thm:kleinsteuber}. We first define the unmixing matrix
  $V\coloneqq A^{-1}$ and introduce the sets of matrices
  \begin{equation*}
    \mathcal{D}_{\operatorname{var}}\coloneqq\{V(\cov(X_k)-\cov(X_l))V^{\top}\,|\,g\in\groups\text{
    and }
    k,l\in g\}.
  \end{equation*}%
  and%
  \begin{equation*}
    \mathcal{D}_{\operatorname{TD}}\coloneqq\{V(\cov(X_k,X_{k-\tau})-\cov(X_l,X_{l-\tau}))V^{\top}\,|\,g\in\groups\text{
    and }
    k,l\in g\}.
  \end{equation*}
  Due to the assumed ICA model and
  Assumption~\ref{assumption:group_structure}, all matrices in the sets
  $\mathcal{D}_{\operatorname{var}}$ and
  $\mathcal{D}_{\operatorname{TD}}$ are diagonal (cf.\
  \eqref{eq:jointdiag_a} and \eqref{eq:jointdiag_b}). Moreover, for $g\in\groups$ and $k,l\in g$ it
  holds that
  \begin{align*}
    &V(\cov(X_k)-\cov(X_l))V^{\top}\\
    &\quad=\cov(S_k)-\cov(S_l)\\
    &\quad=\operatorname{diag}(\var(S_k^1)-\var(S_l^1),\dots,\var(S_k^d)-\var(S_l^d))
  \end{align*}
  and
  \begin{align*}
    &V(\cov(X_k,X_{k-\tau})-\cov(X_l,X_{l-\tau}))V^{\top}\\
    &\quad=\cov(S_k,S_{k-\tau})-\cov(S_l,S_{l-\tau})\\
    &\quad=\operatorname{diag}(\cov(S_k^1,S_{k-\tau}^1)-\cov(S_l^1,S_{l-\tau}^1),\dots,\cov(S_k^d,S_{k-\tau}^d)-\cov(S_l^d,S_{l-\tau}^d)).
  \end{align*}
  Next, we define for all $j\in\{1,\dots,d\}$ the vectors
  \begin{align*}
    \mathbf{z}_j&=\left(\left(\var(S_k^j)-\var(S_l^j)\right)_{k,l\in
        g}\right)_{g\in\groups}\\
    \text{or}\quad
    \mathbf{z}_j&=\left(\left(\cov(S_k^j,S_{k-\tau}^j)-\cov(S_l^j,S_{l-\tau}^j)\right)_{k,l\in g}\right)_{g\in\groups},
  \end{align*}
  depending on whether a variance signal or time-dependence signal is being considered,
  respectively. Then, Assumption~\ref{assumption:var} or
  Assumption~\ref{assumption:time-dependence} implies for all distinct
  pairs $p,q\in\{1,\dots,d\}$ that
  \begin{equation*}
    \abs{\empcorr(\mathbf{z}_p,\mathbf{z}_q)}
    =\frac{\abs{\mathbf{z}_p\cdot\mathbf{z}_q}}{\norm{\mathbf{z}_p}\norm{\mathbf{z}_q}}
    <1.
  \end{equation*}
  Hence, for either $\mathcal{D}=\mathcal{D}_{\operatorname{var}}$ or
  $\mathcal{D}=\mathcal{D}_{\operatorname{TD}}$ it holds that
  $\rho(\mathcal{D})<1$, where $\rho$ is defined in
  \eqref{eq:rho_def}. Since the identity matrix satisfies that
  $\identity D \identity^{\top}$ is diagonal for all $D\in\mathcal{D}$, we
  can invoke Theorem~\ref{thm:kleinsteuber} to conclude that any
  matrix $M\in\R^{d\times d}$ for which $M D M^{\top}$ is diagonal for
  all $D\in\mathcal{D}$, is equal to the identity matrix up to scaling
  and permutation of its columns. Next, we consider the two signal
  types separately.
  \begin{itemize}
  \item\textbf{variance signal:} If there is a variance signal that satisfies
    Assumption~\ref{assumption:var}, assume there exists an invertible
    matrix $\widetilde{A}$ such that for all $g\in\groups$ and
    all $k,l\in g$ it holds that
    \begin{equation*}
      \widetilde{A}^{-1}(\cov(X_k)-\cov(X_l))(\widetilde{A}^{-1})^{\top}=\cov(S_k)-\cov(S_l).
    \end{equation*}
    Then, it also holds that
    \begin{equation*}
      (V\widetilde{A})\underbrace{(\cov(S_k)-\cov(S_l))}_{\in \mathcal{D}_{\operatorname{var}}}(V\widetilde{A})^{\top}=V(\cov(X_k)-\cov(X_l))V^{\top},
    \end{equation*}
    which is diagonal.
  \item \textbf{time-dependence signal:} If there is a time-dependence
    signal that satisfies Assumption~\ref{assumption:time-dependence}, assume
    there exists an invertible matrix $\widetilde{A}$ such that
    for all $g\in\groups$ and all $k,l\in g$ it holds that
    \begin{equation*}
      \widetilde{A}^{-1}(\cov(X_k,X_{k-\tau}))-\cov(X_l,X_{l-\tau}))(\widetilde{A}^{-1})^{\top}=\cov(S_k,S_{k-\tau})-\cov(S_l,S_{l-\tau}).
    \end{equation*}
    Then, it also holds that
    \begin{equation*}
      (V\widetilde{A})\underbrace{(\cov(S_k,S_{k-\tau})-\cov(S_l,S_{l-\tau}))}_{\mathcal{D}_{\operatorname{TD}}}(V\widetilde{A})^{\top}=V(\cov(X_k,X_{k-\tau}))-\cov(X_l,X_{l-\tau}))V^{\top},
    \end{equation*}
    which is diagonal.
  \end{itemize}
  Using the above reasoning, either of the two cases---depending on whether Assumption~\ref{assumption:var} or \ref{assumption:time-dependence} holds---shows that
  $V\widetilde{A}$ is equal to the identity matrix up to permutation
  and rescaling of its columns. Moreover, this implies that
  $\widetilde{A}$ is equal to $A$ up to scaling and permutation of its
  columns. This completes the proof of Theorem~\ref{thm:mixing}.
\end{proof}

\clearpage

\section{Complementary Material}\label{sec:complementary}

\subsection{America's Got Talent}

\begin{figure}[h]
  \centering
  \tikzstyle{mic} = [draw, minimum width = 0.75cm, minimum
  height = 0.75cm, rounded corners=0.2cm, fill=colD!50]
  \tikzstyle{speaker} = [draw, circle, minimum size=.5cm, node
  distance=1.75cm, fill=colC!50]
  \tikzstyle{audience} = [draw, rectangle, minimum width = 3cm, minimum
  height = 6cm, fill=colB!50]
  \tikzstyle{window} = [draw, rectangle, minimum width = 4cm, minimum
  height = 0.75cm, fill=colA!50]
  \begin{tikzpicture}[framed, background
    rectangle/.style={thick, draw=black,
      rounded corners}]
    \node [speaker] (speaker1) at (0,3) {Singer 1};
    \node [speaker] (speaker2) at (0,0) {Singer 2};
    \node [mic] (mic1) at (3.1,2) {Mic 1};
    \node [mic] (mic2) at (3.1,-1) {Mic 2};
    \node [audience] (audience) at (10,0) {Audience};
    \node [window] (window1) at (6,5) {Window 1};
    \node [window] (window2) at (3,-4) {Window 2};
    \draw [colD,domain=-40:10] plot ({1.5*cos(\x)}, {3+1.5*sin(\x)});
    \draw [colD,domain=-40:10] plot ({2*cos(\x)}, {3+2*sin(\x)});
    \draw [colD,domain=-40:10] plot ({2.5*cos(\x)}, {3+2.5*sin(\x)});
    \draw [colD,domain=-40:10] plot ({1.5*cos(\x)}, {0+1.5*sin(\x)});
    \draw [colD,domain=-40:10] plot ({2*cos(\x)}, {0+2*sin(\x)});
    \draw [colD,domain=-40:10] plot ({2.5*cos(\x)}, {0+2.5*sin(\x)});
    \draw [colD,domain=145:215] plot ({10.5+3*cos(\x)}, {3.5*sin(\x)});
    \draw [colD,domain=145:215] plot ({10.5+3.5*cos(\x)}, {4*sin(\x)});
    \draw [colD,domain=145:215] plot ({10.5+4*cos(\x)}, {4.5*sin(\x)});
    \draw [colD,domain=145:215] plot ({10.5+4.5*cos(\x)}, {5*sin(\x)});
    \draw [colD,domain=240:300] plot ({6+3*cos(\x)}, {7+3*sin(\x)});
    \draw [colD,domain=240:300] plot ({6+3.5*cos(\x)}, {7+3.5*sin(\x)});
    \draw [colD,domain=240:300] plot ({6+4*cos(\x)}, {7+4*sin(\x)});
    \draw [colD,domain=60:120] plot ({3+3*cos(\x)}, {-6+3*sin(\x)});
    \draw [colD,domain=60:120] plot ({3+3.5*cos(\x)}, {-6+3.5*sin(\x)});
    \draw [colD,domain=60:120] plot ({3+4*cos(\x)}, {-6+4*sin(\x)});
  \end{tikzpicture}
  \caption{Schematic of the ``America's Got Talent Duet Problem'' described in
    Example~\ref{ex:toy_example}. The sound from the windows and
    audience is taken to be confounding noise which has fixed
    covariance structure over given time blocks. The challenge is to
    recover the sound signals from the individual singers given the
    recordings of the two microphones.}
  \label{fig:crowdedroom}
\end{figure}
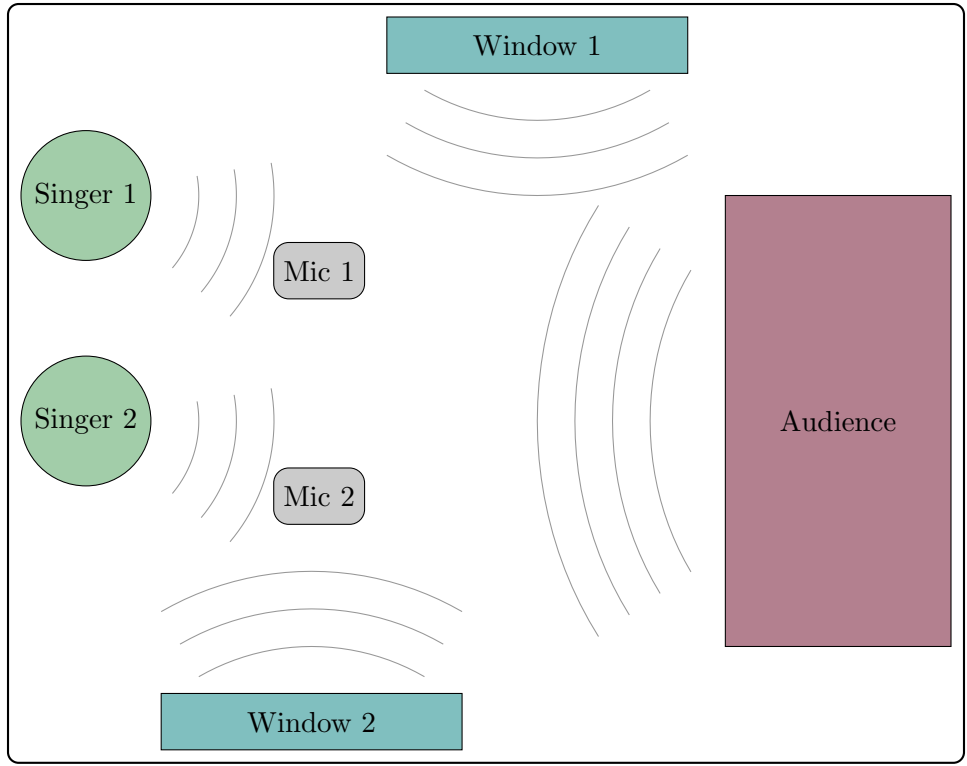

\clearpage

\subsection{Causality
  Example}\label{sec:causal_appendix}

\begin{figure}[h!]
  \centering
  \begin{tikzpicture}[scale=2.5]
    \tikzstyle{VertexStyle} = [shape = circle]
    \SetGraphUnit{2}
    \Vertex[Math,L=\COtwo_t,x=2,y=0.6]{X1}
    \Vertex[Math,L=\Temp_t,x=2,y=-0.6]{X2}
    \Vertex[Math,L=\widetilde{S}^1_t,x=2,y=1.5]{E1}
    \Vertex[Math,L=\widetilde{S}^2_t,x=2,y=-1.5]{E2}
    \Vertex[Math,L=\COtwo_{t-1},x=0,y=0.6]{X11}
    \Vertex[Math,L=\Temp_{t-1},x=0,y=-0.6]{X21}
    \Vertex[Math,L=\widetilde{S}^1_{t-1},x=0,y=1.5]{E11}
    \Vertex[Math,L=\widetilde{S}^2_{t-1},x=0,y=-1.5]{E21}
    \Vertex[Math,L=\COtwo_{t-2},x=-2,y=0.6]{X12}
    \Vertex[Math,L=\Temp_{t-2},x=-2,y=-0.6]{X22}
    \Vertex[Math,L=\widetilde{S}^1_{t-2},x=-2,y=1.5]{E12}
    \Vertex[Math,L=\widetilde{S}^2_{t-2},x=-2,y=-1.5]{E22}
    \tikzstyle{VertexStyle} = [minimum width = 4.2em]
    \tikzstyle{EdgeStyle} = [->,>=stealth',shorten > = 2pt]
    \Edge(E1)(X1)
    \Edge(E2)(X2)
    \Edge(E11)(X11)
    \Edge(E21)(X21)
    \Edge(E12)(X12)
    \Edge(E22)(X22)
    \tikzset{EdgeStyle/.append style = {->, bend left, color=red}}
    \Edge[label=$\alpha$](X1)(X2)
    \Edge[label=$\beta$](X2)(X1)
    \Edge[label=$\alpha$](X11)(X21)
    \Edge[label=$\beta$](X21)(X11)
    \Edge[label=$\alpha$](X12)(X22)
    \Edge[label=$\beta$](X22)(X12)
    \tikzstyle{EdgeStyle} = [->,>=stealth',shorten > = 2pt, color=lightgray]
    \Edge(X12)(X11)
    \Edge(X22)(X21)
    \Edge(X22)(X11)
    \Edge(X12)(X21)
    \Edge(X11)(X1)
    \Edge(X21)(X2)
    \Edge(X21)(X1)
    \Edge(X11)(X2)
    \tikzstyle{EdgeStyle} = [>=stealth',shorten > = 2pt,shorten < = 2pt, <->,
    bend left=40, dashed, color=gray]
    \Edge(E1)(E2)
    \Edge(E11)(E21)
    \Edge(E12)(E22)
  \end{tikzpicture}
  \caption{Graphical representation of the causal feedback model between
    carbon dioxide ($\text{CO}_2$) and temperature (T). The dashed
    line corresponds to stationary confounding.}
  \label{fig:graph}
\end{figure}
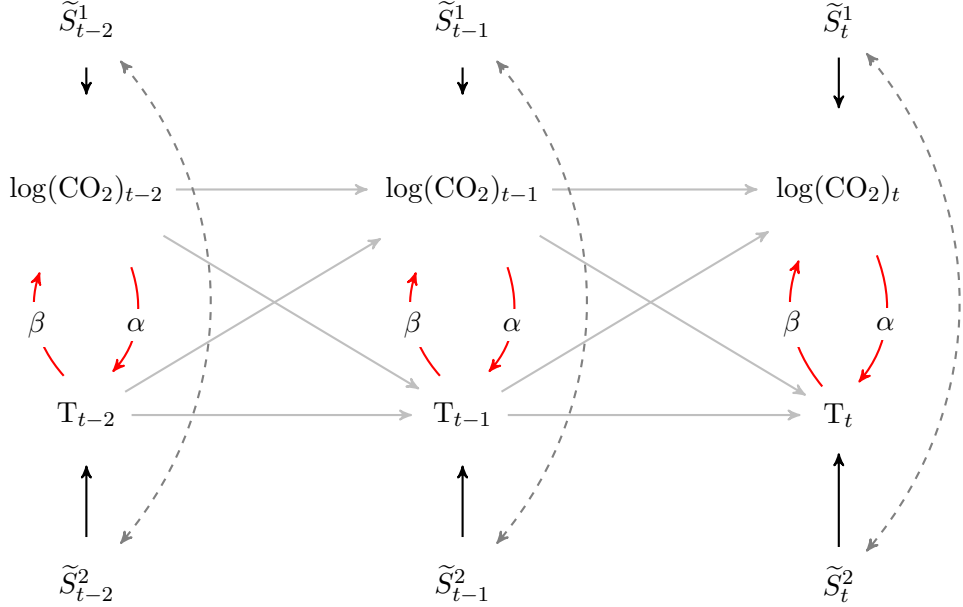

\subsection{Simulations}\label{sec:appendix_simu}

The GARCH model that we simulate from in Section~\ref{sec:GARCH} is
specified as follows. We simulate sources $S^1,\dots,S^d$ from the following
GARCH-type model
\begin{align*}
  \sigma_i^2&=a_1 + a_2\cdot(S_{i-1}^j)^2 + a_3\cdot\sigma_{i-1}^2\\
  S_{i}^j&=b_1S_{i-1}^j +\cdots + b_pS_{i-p}^j + \sigma_i\epsilon_i,
\end{align*}
where the $\epsilon_i$ are independent and standard normal. Moreover, the noise terms
$H^1,\dots,H^d$ are assumed to be either given by the following AR-process
\begin{equation*}
  H_{i}^j=c_1H_{i-1}^j +\cdots + c_qH_{i-q}^j + \nu_i,
\end{equation*}
where $\nu_i$ are independent standard normal, $q$ is uniformly
distributed on $\{1,\dots,10\}$ and $c_i$ independent
$\mathcal{N}\left(0, 1/(i+1)^2\right)$ or simply as iid
$\mathcal{N}\left(0, 1\right)$ random variables. The final data is
then constructed according to the following equation
\begin{equation*}
  X_i=A\cdot S_{i}+\widetilde{H}_{i},
\end{equation*}
where $\widetilde{H}_i=ACH_i$ and $A, C\in\R^{d\times d}$ are
sampled with iid entries from $\mathcal{N}(0, 1)$ and
$\mathcal{N}(0, \frac{1}{d})$, respectively. To illustrate, the
effect of the signal type we consider the following three settings.
\begin{itemize}
\item \textbf{Setting 1 (time-independent with
    changing variance)}\\
  Set $a=(0.005, 0.026, 0.97)$ such that the variance changes over
  time and $p=0$ to ensure time-independent signals. Based on these
  settings we sample $n=200000$ observations.
\item \textbf{Setting 2 (varying time-dependence structure
   with constant variance)}\\
  Set $a=(1, 0, 0)$ such that the variance is fixed to $1$. Then,
  sample $p$ $100$ times uniformly from $\{1,\dots,10\}$ and $b_i$
  independent from $\mathcal{N}\left(0, 1/(i+1)^2\right)$ and simulate
  $2000$ observations for each of the $100$ parameter settings,
  leading to a total of $n=200000$ observations.
\item \textbf{Setting 3 (varying time-dependence structure
    with changing variance)}\\
  Set $a=(0.005, 0.026, 0.97)$ such that the variance changes over
  time. Then, we sample $p$ $100$ times uniformly from
  $\{1,\dots,10\}$ and $b_i$ independent from
  $\mathcal{N}\left(0, 1/(i+1)^2\right)$ and simulate $2000$
  observations for each of the $100$ parameter settings, leading to
  a total of $n=200000$ observations.
\end{itemize}

\clearpage

\section{EEG Experiments on \protect\hyperlink{dat:bcicomp}{Data~Set~4}}\label{sec:EEG_results}

Analogous to Sections~\ref{sec:stability} and
\ref{sec:EEG.classification} we conducted experiments on the
BCICompIV2a \hyperlink{dat:bcicomp}{Data~Set~4}, the results of which
are presented in the subsequent sections.

\begin{mdframed}[roundcorner=5pt,
  frametitle={\hypertarget{dat:bcicomp}{Data Set 4}: BCICompIV2a data}]
  This data set is due to \citet[Section 5]{tangermann2012}
  and consists of EEG recordings of 9 subjects performing multiple
  trials of 4 different motor imagery tasks.
  The data set contains recordings of
  \begin{itemize}
  \item 9 subjects, each recorded on 2 different days,
  \item for each subject and day there exist 6 runs with 48 trials,
  \item each recording consists of 22 EEG channels recorded at 250 Hz sampling frequency,
  \item and is bandpass filtered between 0.5 and 100 Hz and is 50 Hz
    notch filtered.
  \end{itemize}
  For our analysis we only use the trial-data, i.e., the concatenated
  segments of seconds 3--6 of each trial (corresponding to the motor
  imagery part of the trials~\citep{tangermann2012}).  We further
  preprocess the data by re-referencing to common average reference
  (car) and projecting onto the orthogonal complement of the null
  component.

  As features for classification experiments (cf.\
  Section~\ref{sec:EEG.classification}) on this data set we use
  bandpower in the 8--30 Hz band as measured by the log-variance of
  the 8--30 Hz bandpass-filtered trial data~\citep{lotte2018review}.
\end{mdframed}

\clearpage
\subsection{Stability and Independence}\label{sec:appendix_stability}

\begin{figure}[h]
\centering
  \scalebox{0.84}{\includegraphics{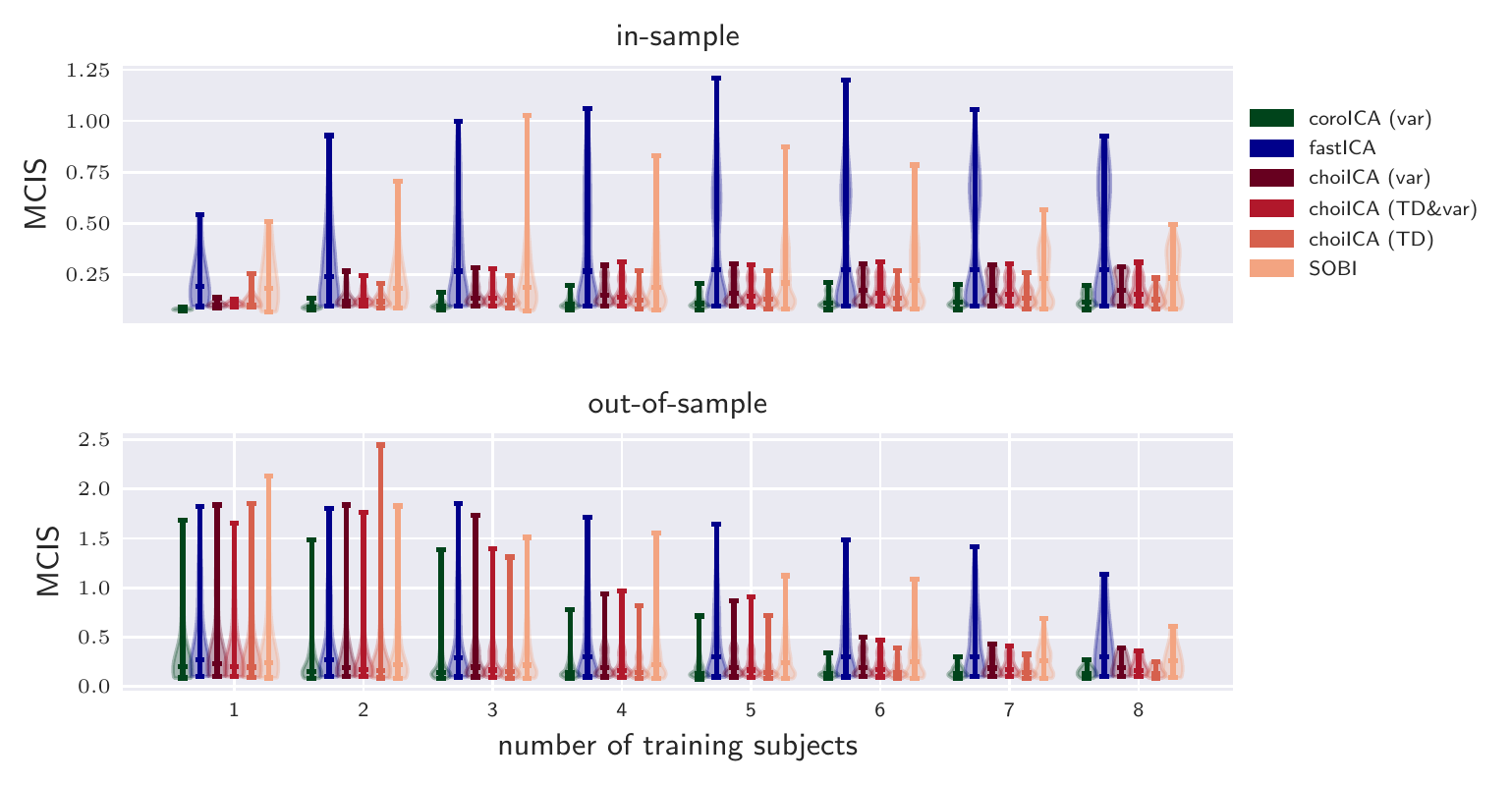}}
  \caption{Experimental results for comparing the stability of sources
    (MCIS: small implies stable) trained on different numbers of
    training subjects (cf.\ Section~\ref{sec:stability}),
    here on the BCICompIV2a \protect\hyperlink{dat:bcicomp}{Data~Set~4},
    demonstrating that \coroICA, in contrast to the competing ICA methods,
    can successfully incorporate more
    training subjects to learn more stable unmixing matrices when
    applied to new unseen subjects.}
  \label{fig:instability_bci}
\end{figure}

\begin{figure}[h]
\centering
  \scalebox{0.84}{\includegraphics{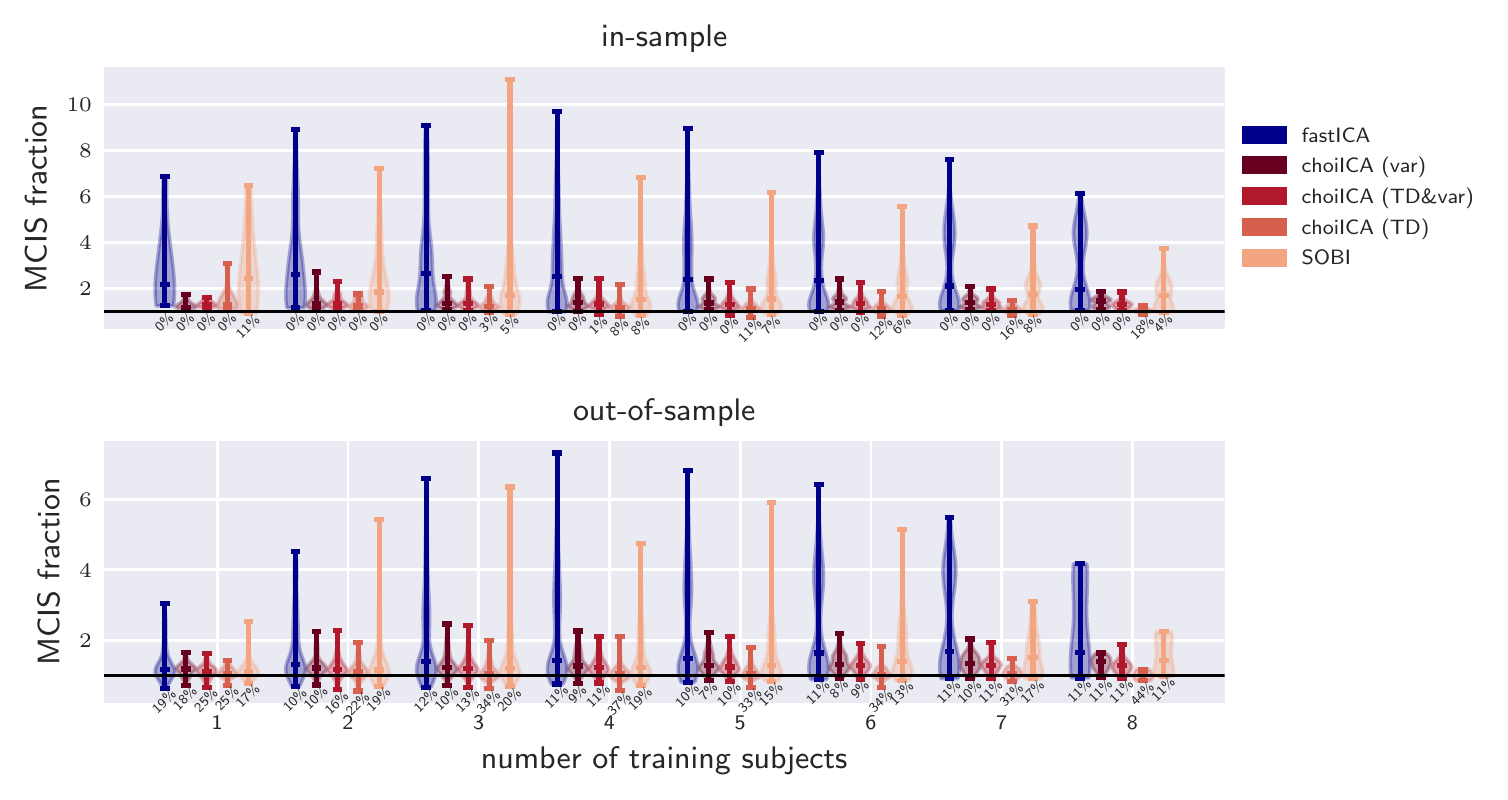}}
  \caption{Experimental results for comparing the stability of sources
    of competitors relative to the stability obtained by
    \coroICA (MCIS fraction: $>1$ implies less stable than
    \coroICA) trained on different numbers of training subjects (cf.\
    Section~\ref{sec:stability}), here on the BCICompIV2a
    \protect\hyperlink{dat:bcicomp}{Data~Set~4}, demonstrating that \coroICA
    can successfully incorporate more training subjects to learn more
    stable unmixing matrices when applied to new unseen subjects.}
  \label{fig:instability_fraction_bci}
\end{figure}

\FloatBarrier

\subsection{Classification based on Recovered Sources}\label{sec:appendix_EEG.classification}

\begin{figure}[h!]
  \includegraphics{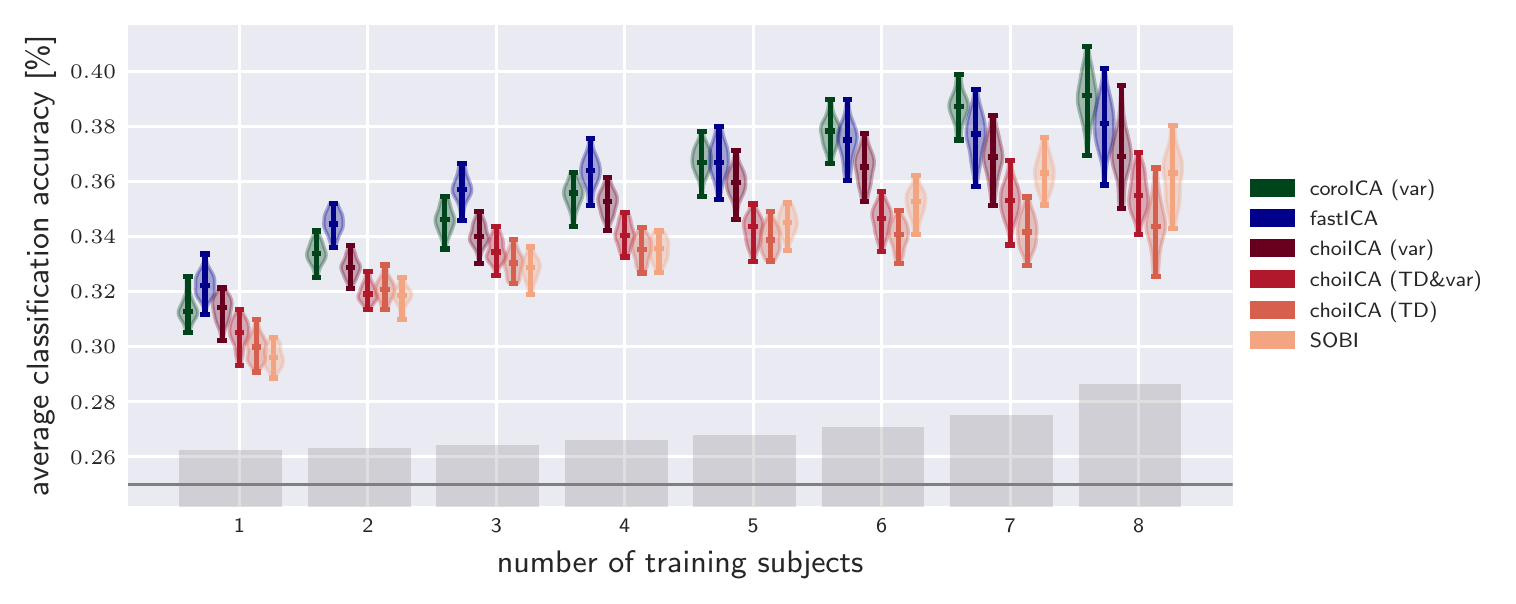}
  \caption{Classification accuracies on held-out subjects (cf.\
    Section~\ref{sec:EEG.classification}), here on the BCICompIV2a
    \protect\hyperlink{dat:bcicomp}{Data~Set~4}. Gray regions indicate a 95\%
    confidence interval of random guessing accuracies.}
  \label{fig:classification_bci}
\end{figure}

\clearpage
\section{All Topographies and Activation Maps on \protect\hyperlink{dat:covertattention}{Data~Set~3}}\label{sec:alltopos}
\begin{figure}[h!]
  \scalebox{0.94}{\includegraphics{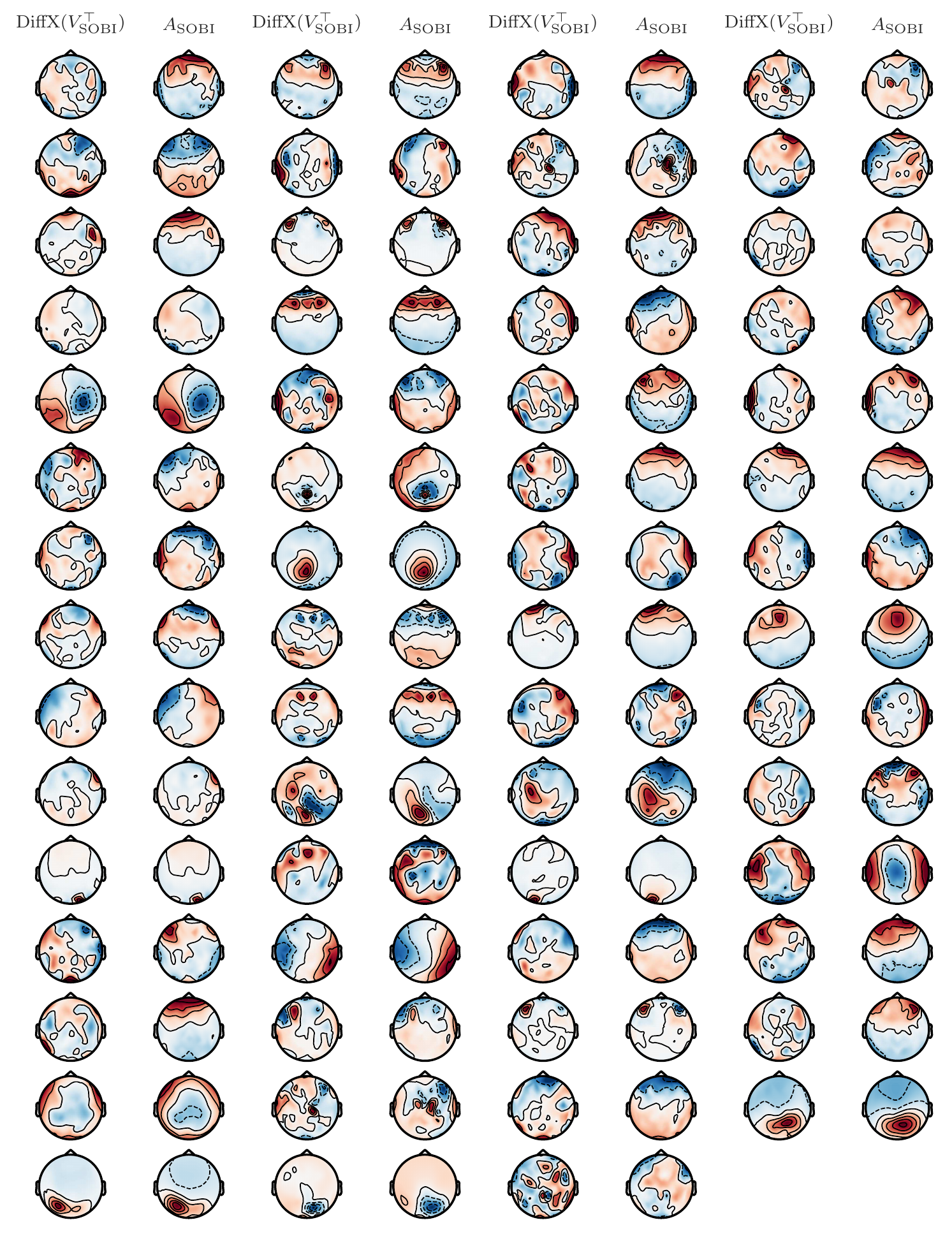}}
  \caption{Activation maps (left of each pair of columns) and
    topographies (right of each pair of columns) of $59$ components
    extracted by SOBI on the CovertAttention
    \protect\hyperlink{dat:covertattention}{Data~Set~3}.  For
    components that are well resolved, both should look similar
    (cf.~Section~\ref{sec:scores} and \ref{sec:topographic_maps}).  }
\end{figure}

\begin{figure}[h!]
  \scalebox{0.94}{\includegraphics{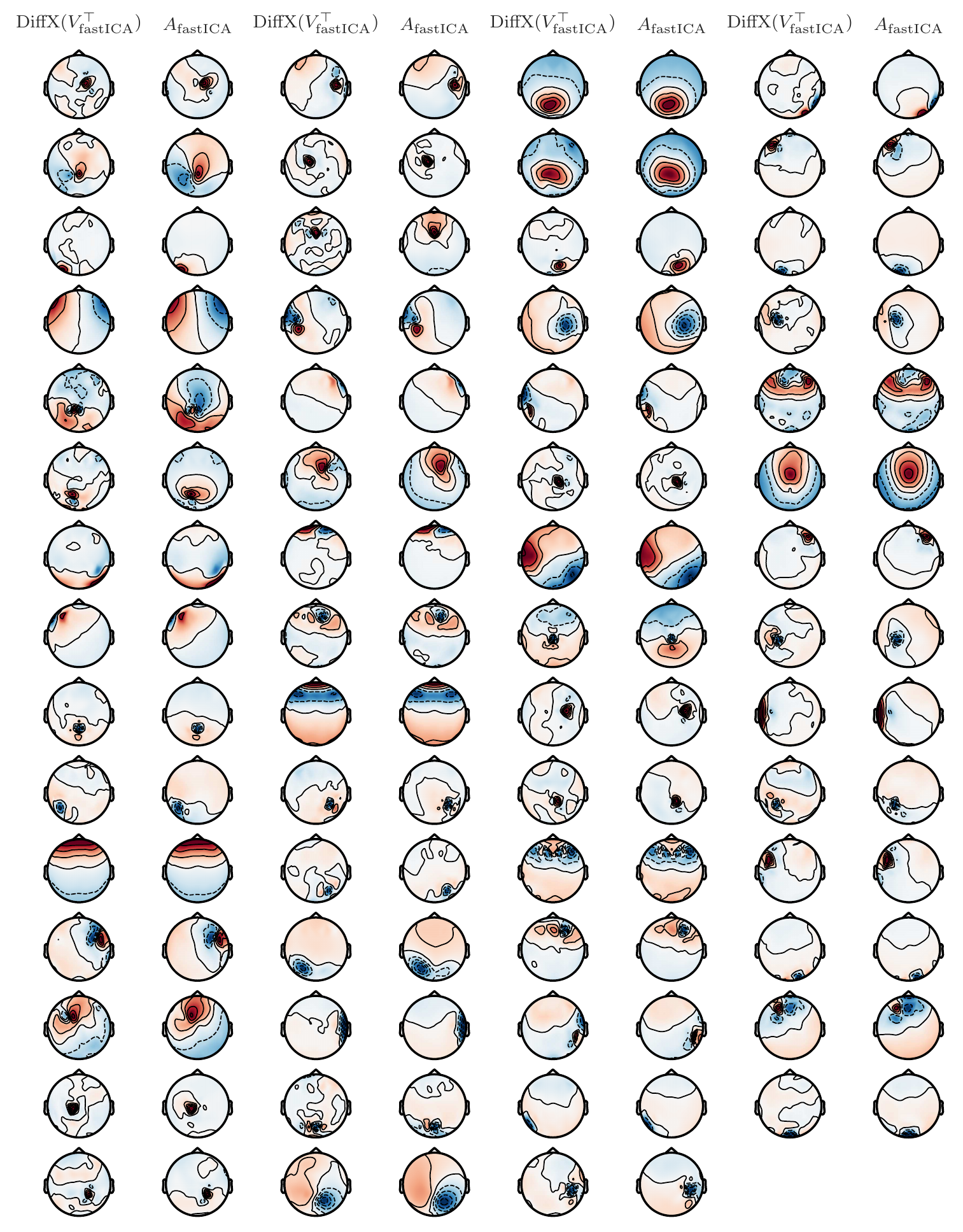}}
  \caption{Activation maps (left of each pair of columns) and
    topographies (right of each pair of columns) of $59$ components
    extracted by fastICA on the CovertAttention
    \protect\hyperlink{dat:covertattention}{Data~Set~3}.  For
    components that are well resolved, both should look similar
    (cf.~Section~\ref{sec:scores} and \ref{sec:topographic_maps}).  }
\end{figure}

\begin{figure}[h!]
  \scalebox{0.94}{\includegraphics{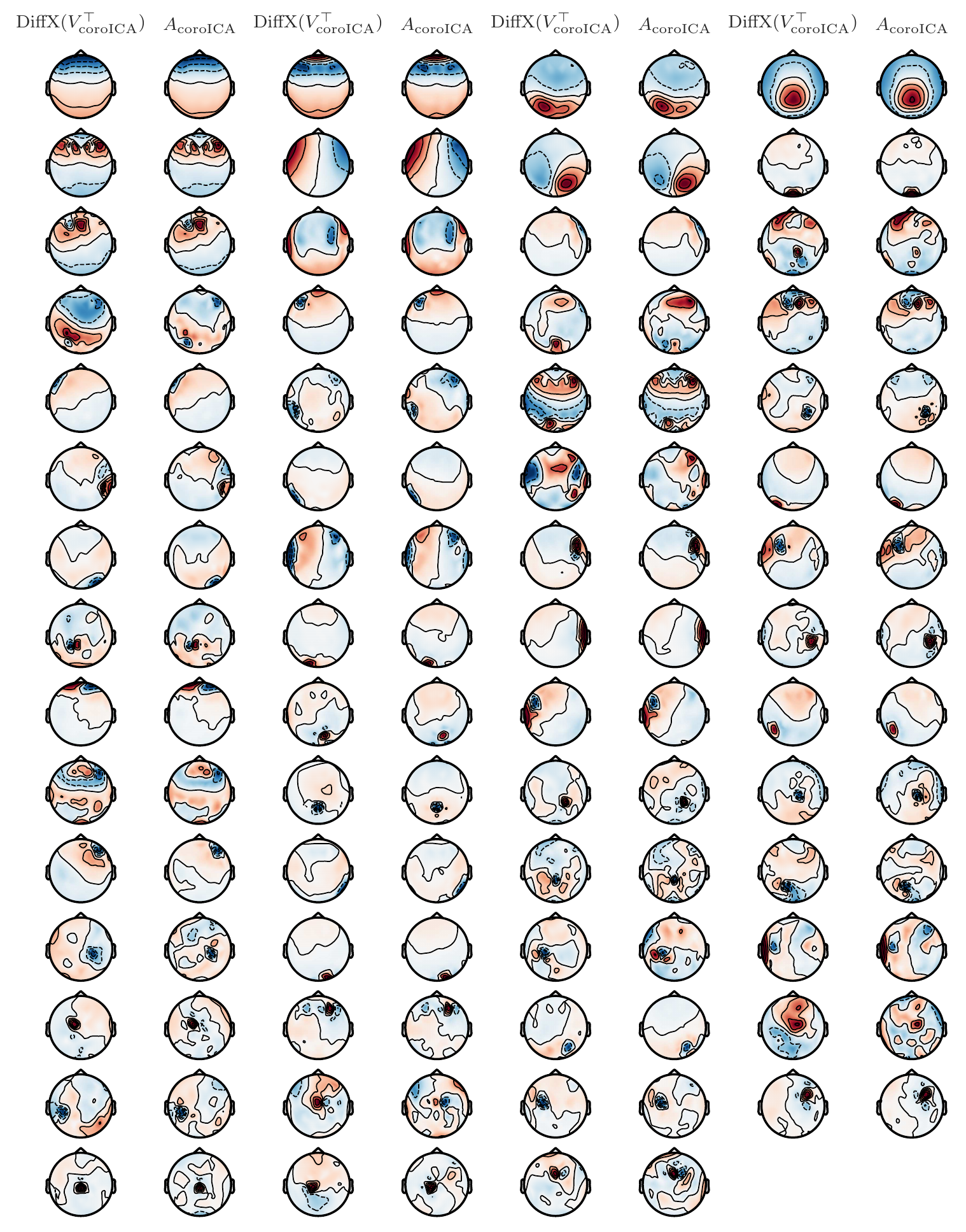}}
  \caption{Activation maps (left of each pair of columns) and
    topographies (right of each pair of columns) of $59$ components
    extracted by \coroICAvar on the CovertAttention
    \protect\hyperlink{dat:covertattention}{Data~Set~3}.  For
    components that are well resolved, both should look similar
    (cf.~Section~\ref{sec:scores} and \ref{sec:topographic_maps}).  }
\end{figure}

\FloatBarrier

\bibliography{references}

\end{document}